\pdfoutput=1
\documentclass{article}
\usepackage[preprint]{nips_2018}

\usepackage[utf8]{inputenc} 
\usepackage[T1]{fontenc}    
\usepackage{graphicx}
\usepackage{subfigure}	
\usepackage{booktabs} 
\usepackage{amsmath,amssymb,amsthm,amsfonts,latexsym}        
\usepackage{hyperref}
\usepackage{url}            
\usepackage{nicefrac}       
\usepackage{microtype}      
\usepackage{setspace}
\usepackage{algorithm}
\usepackage[noend]{algorithmic}
\usepackage{dsfont}

\usepackage[usenames,dvipsnames]{xcolor}
\usepackage[capitalize,noabbrev]{cleveref}

\title{New Insights into Bootstrapping for Bandits}
\author{
Sharan Vaswani\\
University of British Columbia\\
\texttt{sharanv@cs.ubc.ca}\\
\And
Branislav Kveton \\
Adobe Research \\
\texttt{kveton@adobe.com} \\
\And
Zheng Wen \\
Adobe Research \\
\texttt{zwen@adobe.com} \\
\AND
Anup Rao\\
Adobe Research \\
\texttt{anuprao@adobe.com} \\
\And
Mark Schmidt \\
University of British Columbia \\
\texttt{schmidtm@cs.ubc.ca} \\
\And
Yasin Abbasi-Yadkori\\
Adobe Research \\
\texttt{abbasiya@adobe.com} \\
}
\DeclareMathOperator*{\argmax}{arg\,max}

\newtheorem{proposition}{Proposition}
\newtheorem{lemma}{Lemma}
\newtheorem{theorem}{Theorem}

\newcommand{\E}{\mathbb{E}}
\newcommand{\I}{\mathbb{I}}
\usepackage{xspace}
\newcommand{\hth}{\hat{\theta}}
\newcommand{\tth}{\tilde{\theta}}

\newcommand{\GWB}{WB\xspace}
\newcommand{\NPB}{NPB\xspace}
\newcommand{\cD}{\mathcal{D}}
\newcommand{\cG}{\mathcal{H}}
\newcommand{\cH}{\mathcal{H}}
\newcommand{\cI}{\mathcal{I}}
\newcommand{\cP}{\mathcal{P}}
\newcommand{\cL}{\mathcal{L}}
\newcommand{\cT}{\mathcal{T}}
\newcommand{\bw}{{\bf w}}
\newcommand{\bx}{{\bf x}}
\newcommand{\by}{{\bf y}}

\begin{document}

\maketitle

\begin{abstract}
We investigate the use of bootstrapping in the bandit setting. We first show that the commonly used non-parametric bootstrapping (\NPB) procedure can be provably inefficient and establish a near-linear lower bound on the regret incurred by it under the bandit model with Bernoulli rewards. We show that \NPB with an appropriate amount of forced exploration can result in sub-linear albeit sub-optimal regret. As an alternative to \NPB, we propose a \emph{weighted bootstrapping} (\GWB) procedure. For Bernoulli rewards, \GWB with multiplicative exponential weights is mathematically equivalent to Thompson sampling (TS) and results in near-optimal regret bounds. Similarly, in the bandit setting with Gaussian rewards, we show that \GWB with additive Gaussian weights achieves near-optimal regret. Beyond these special cases, we show that \GWB leads to better empirical performance than TS for several reward distributions bounded on $[0,1]$. For the contextual bandit setting, we give practical guidelines that make bootstrapping simple and efficient to implement and result in good empirical performance on real-world datasets.
\end{abstract}

\section{Introduction}
\label{sec:introduction}
The multi-armed bandit framework~\cite{lai1985asymptotically,bubeck2012regret,auer2002using,auer2002finite} is a classic approach for sequential decision-making under uncertainty. The basic framework consists of independent \emph{arms} that correspond to different choices or actions. These may be different treatments in a clinical trial or different products that can be recommended to the users of an online service. Each arm has an associated expected \emph{reward} or utility. Typically, we do not have prior information about the utility of the available choices and the \emph{agent} learns to make ``good'' decisions via repeated interaction in a trial-and-error fashion. Under the bandit setting, in each interaction or \emph{round}, the agent selects an arm and observes a reward \emph{only} for the selected arm. The objective of the agent is to maximize the reward accumulated across multiple rounds. This results in an \emph{exploration-exploitation trade-off}: exploration means choosing an arm to gain more information about it, while exploitation corresponds to choosing the arm with the highest estimated reward so far. The \emph{contextual bandit} setting~\cite{wang2005bandit,pandey2007multi,kakade2008efficient,dani2008stochastic,li2010contextual,agrawal2013thompson} is a generalization of the bandit framework and assumes that we have additional information in the form of a feature vector or ``context'' at each round. A context might be used to encode the medical data of a patient in a clinical trial or the demographics of an online user of a recommender system. In this case, the expected reward for an arm is an unknown function~\footnote{We typically assume a parametric form for this function and infer the corresponding parameters from observations.} of the context at that particular round. For example, for \emph{linear bandits}~\cite{rusmevichientong2010linearly,dani2008stochastic,abbasi2011improved}, this function is assumed to be linear implying that the expected reward can be expressed as an inner product between the context vector and an (unknown) parameter to be learned from observations. 

In both the bandit and contextual bandit settings, there are three main strategies for addressing the exploration-exploitation tradeoff: (i) $\epsilon$-greedy~\cite{langford2008epoch} (ii) optimism-in-the-face-of-uncertainty~\cite{auer2002using,abbasi2011improved} (OFU) and (iii) Thompson sampling~\cite{agrawal2013thompson}. Though $\epsilon$-greedy (EG) is simple to implement and is widely used in practice, it results in sub-optimal performance from a theoretical stand-point. In practice, its performance heavily relies on choosing the right exploration parameter and the strategy for annealing it. Strategies based on optimism under uncertainty rely on constructing confidence sets and are statistically optimal and computationally efficient in the bandit~\cite{auer2002finite} and linear bandit~\cite{abbasi2011improved} settings. However, for non-linear feature-reward mappings, we can construct only approximate confidence sets~\cite{filippi2010parametric,li2017provable,zhang2016online,jun2017scalable} that result in over-conservative uncertainty estimates~\cite{filippi2010parametric} and consequently to worse empirical performance. Given a prior distribution over the rewards or parameters being inferred, Thompson sampling (TS) uses the observed rewards to compute a posterior distribution. It then uses samples from the posterior to make decisions. TS is computationally efficient when we have a closed-form posterior like in the case of Bernoulli or Gaussian rewards. For reward distributions beyond those admitting conjugate priors or for complex non-linear feature-reward mappings, it is not possible to have a closed form posterior or obtain exact samples from it. In these cases, we have to rely on computationally-expensive approximate sampling techniques~\cite{riquelme2018deep}. 

To address the above difficulties, bootstrapping~\cite{efron1992bootstrap} has been used in the bandit~\cite{baransi2014sub,eckles2014thompson}, contextual bandit~\cite{tang2015personalized,elmachtoub2017practical} and deep reinforcement learning~\cite{osband2015bootstrapped,osband2016deep} settings. All previous work uses \emph{non-parametric bootstrapping} (explained in Section~\ref{sec:npb-proc}) as an approximation to TS. As opposed to maintaining the entire posterior distribution for TS, bootstrapping requires computing only point-estimates (such as the maximum likelihood estimator). Bootstrapping thus has two major advantages over other existing strategies: (i) Unlike OFU and TS, it is simple to implement and does not require designing problem-specific confidence sets or efficient sampling algorithms. (ii) Unlike EG, it is not sensitive to hyper-parameter tuning. In spite of its advantages and good empirical performance, bootstrapping for bandits is not well understood theoretically, even under special settings of the bandit problem. Indeed, to the best of our knowledge,~\cite{elmachtoub2017practical} is the only work that attempts to theoretically analyze the non-parametric bootstrapping (referred to as \NPB) procedure. For the bandit setting with Bernoulli rewards and a Beta prior (henceforth referred to as the Bernoulli bandit setting), they prove that both TS and \NPB will take similar actions as the number of rounds increases. However, this does not have any implication on the regret for \NPB.  

In this work, we first show that the \NPB procedure used in the previous work is provably inefficient in the Bernoulli bandit setting (Section~\ref{sec:lb}). In particular, we establish a near-linear lower bound on the incurred regret. In Section~\ref{sec:forced}, we show that \NPB with an appropriate amount of forced exploration (done in practice in~\cite{elmachtoub2017practical,tang2015personalized}) can result in a sub-linear though sub-optimal upper bound on the regret. As an alternative to \NPB, we propose the \emph{weighted bootstrapping} (abbreviated as \GWB) procedure. For Bernoulli (or more generally categorical) rewards, we show that \GWB with multiplicative exponential weights is mathematically equivalent to TS and thus results in near-optimal regret. Similarly, for Gaussian rewards, \GWB with additive Gaussian weights is equivalent to TS with an uninformative prior and also attains near-optimal regret. 

In Section~\ref{sec:experiments}, we empirically show that for several reward distributions on $[0,1]$, \GWB outperforms TS with a randomized rounding procedure proposed in~\cite{agrawal2013thompson}. In the contextual bandit setting, we give two implementation guidelines. To improve the \emph{computational efficiency} of bootstrapping, prior work~\cite{eckles2014thompson,elmachtoub2017practical,tang2015personalized} approximated it by an ensemble of models that requires additional hyperparameter tuning, such as choosing the size of the ensemble; or problem-specific heuristics, for example \cite{elmachtoub2017practical} uses a lazy update procedure specific to decision trees. We find that with appropriate stochastic optimization, bootstrapping (without any approximation) is computationally efficient and simple to implement. Our second guideline is for the  initialization of the bootstrapping procedure. Prior work~\cite{elmachtoub2017practical,tang2015personalized} used \emph{forced exploration} at the beginning of bootstrapping, by pulling each arm for some number of times or by adding pseudo-examples. This involves tuning additional hyper-parameters, for example, \cite{elmachtoub2017practical} pull each arm $30$ times before bootstrapping. Similarly, the number of pseudo-examples or the procedure for generating them is rather arbitrary. We propose a simple method for generating such examples and experimentally validate that using $O(d)$ pseudo-examples, where $d$ is the dimension of the context vector, leads to consistently good performance. These contributions result in a simple and efficient implementation of the bootstrapping procedure. We experimentally evaluate bootstrapping with several parametric models and real-world datasets.

\section{Background}
We describe the framework for the contextual bandit problem in Section~\ref{sec:prob-def}. In Section~\ref{sec:bootstrapping}, we give the necessary background on bootstrapping and then explain its adaptation to bandits in Section~\ref{sec:bb}.

\subsection{Bandits Framework}
\label{sec:prob-def}
The bandit setting consists of $K$ arms where each arm $j$ has an underlying (unknown) reward distribution. The protocol for a bandit problem is as follows: in each round $t = 1 \dots T$, the bandit algorithm selects an arm $j_t$. It then receives a reward $r_t$ sampled from the underlying reward distribution for the selected arm $j_t$. The \emph{best} or optimal arm is defined as the one with the highest \emph{expected} reward. The aim of the bandit algorithm is to maximize the expected cumulative reward,
or alternatively, to minimize the \emph{expected cumulative regret}. The cumulative regret $R(T)$ is the cumulative loss in the reward across $T$ rounds because of the lack of knowledge of the optimal arm. 

In the contextual bandit setting~\cite{langford2008epoch,li2017provable,chu2011contextual}, the expected reward at round $t$ depends on the \emph{context} or feature vector $\bx_t \in \mathbb{R}^d$. Specifically, each arm $j$ is parametrized by the (unknown) vector $\theta^{*}_{j} \in \mathbb{R}^d$ and its expected reward at round $t$ is given by $m(x_t,\theta^*_j)$ i.e. $\E[r_t | j_t = j] = m(x_t, \theta_j^*)$. Here the function $m(\cdot, \cdot): \mathbb{R}^{d} \times \mathbb{R}^{d} \rightarrow \mathbb{R}$ is referred to as the \emph{model class}. Given these definitions, the expected cumulative regret is defined as follows:
\begin{align}
\E[R(T)] = \sum_{t = 1}^{T} \left[ \max_{j} \left[ m(x_t, \theta_j^*) \right] - m(x_t, \theta_{j_t}^*) \right]
\label{eq:cb-cumu-regret}
\end{align}
The standard bandit setting is a special case of the above framework. To see this, if $\mu_j$ denotes the expected reward of arm $j$ of the $K$-arm bandit, then it can be obtained by setting $d=1$, $x_t=1$ for all $t$ and 
$m(x_t, \theta_j^*) =\theta_j^* = \mu_j$ for all $j$. Assuming that arm $1$ is the optimal arm, i.e.  $\mu_1 \geq \mu_j$ $\forall j$, then the expected cumulative regret is defined as: $\E[R(T)] = T \mu_1 - \sum_{t = 1}^{T} \E \left[ r_{t} \right]$. Throughout this paper, we describe our algorithm under the general contextual bandit framework, but develop our theoretical results under the simpler bandit setting.
 %
%
\subsection{Bootstrapping}
\label{sec:bootstrapping}
In this section, we set up some notation and describe the bootstrapping procedure in the offline setting. Assume we have a set of $n$ data-points denoted by $\cD = \{\bx_{i}, y_{i}\}$. Here, $\bx_{i}$ and $y_{i}$ refer to the feature vector and observation (alternatively label) for the $i^{th}$ point. We assume a parametric generative model (parametrized by $\theta$) from the features $\bx$ to the observations $y$. Given $\cD$, the log-likelihood of observing the data is given by $\cL(\theta) = \sum_{i \in \cD} \log \left[ \cP( y_i | x_i, \theta ) \right]$ where $\cP( y_i | x_i, \theta )$ is the probability of observing label $y_i$ given the feature vector $\bx_i$, under the model parameters $\theta$. In the absence of features, the probability of observing $y_i$ (for all $i \in [n]$) is given by $\cP(y_i | \theta)$. The maximum likelihood estimator (MLE) for the observed data is defined as $\hth \in \argmax_{\theta} \cL(\theta)$. In this paper, we mostly focus on Bernoulli observations without features in which case, $\hth = \frac{\sum_{i \in \cD} y_{i}}{n}$. 

Bootstrapping is typically used to obtain uncertainty estimates for a model fit to data. The general bootstrapping procedure consists of two steps: (i) Formulate a \emph{bootstrapping log-likelihood} function $\tilde{\cL}(\theta, Z)$ by injecting stochasticity into $\cL(\cdot)$ via the random variable $Z$ such that $\E_{Z} \left[ \tilde{\cL}(\theta, Z) \right] = \cL(\theta)$. (ii) Given $Z = z$, generate a \emph{bootstrap sample} $\tth$ as: $\tth \in \argmax_{\theta} \tilde{\cL}(\theta, z)$. In the offline setting~\cite{friedman2001elements}, these steps are repeated $B$ (usually $B = 10^4$) times to obtain the set $\{ \tth^{1}, \tth^{2}, \ldots \tth^{B} \}$. The variance of these samples is then used to estimate the uncertainty in the model parameters $\hth$. Unlike a Bayesian approach that requires characterizing the entire posterior distribution in order to compute uncertainty estimates, bootstrapping only requires computing point-estimates (maximizers of the bootstrapped log-likelihood functions). In Sections~\ref{sec:npb} and~\ref{sec:gen-boot}, we discuss two specific bootstrapping procedures. 

\subsection{Bootstrapping for Bandits}
\label{sec:bb}
\begin{algorithm}[ht]
\begin{algorithmic}[1]
   \STATE{\textbf{Input}: $K$ arms, Model class $m$}  
   \STATE Initialize history: $\forall j \in [K]$, $\cD_{j} = \{\}$
    \FOR{$t=1$ {\bfseries to} $T$}    
	    \STATE Observe context vector $\bx_{t}$    
    	\STATE For all $j$, compute bootstrap sample $\tth_{j}$ \label{eq:boot-sample} (According to Sections~\ref{sec:npb} and~\ref{sec:gen-boot})
		\STATE Select arm: $j_t = \argmax_{j \in [K]} m( \bx_{t}, \tth_{j})$
		\label{eq:arm-sel}		
		\STATE Observe reward $r_{t}$ 
   		\label{eq:get-obs}    		
		\STATE Update history: $\cD_{j_{t}} = \cD_{j_{t}} \cup \{ \bx_{t}, r_{t} \}$
		\label{eq:update}	    
	\ENDFOR
\end{algorithmic}
\caption{Bootstrapping for contextual bandits}
\label{algo:bb}
\end{algorithm}
In the bandit setting, the work in~\cite{eckles2014thompson,tang2015personalized,elmachtoub2017practical} uses bootstrapping as an approximation to Thompson sampling (TS). The basic idea is to compute \emph{one bootstrap sample} and treat it as a sample from an underlying posterior distribution in order to emulate TS. In Algorithm~\ref{algo:bb}, we describe the procedure for the contextual bandit setting. At every round $t$, the set $\cD_{j}$ consists of the features and observations obtained on pulling arm $j$ in the previous rounds. The algorithm (in line $5$) uses the set $\cD_{j}$ to compute a bootstrap sample $\tth_{j}$ for each arm $j$. Given the bootstrap sample for each arm,  the algorithm (similar to TS) selects the arm $j_t$ maximizing the reward conditioned on this bootstrap sample (line $6$). After obtaining the observation (line $7$), the algorithm updates the set of observations for the selected arm (line $8$). In the subsequent sections, we instantiate the procedures for generating the bootstrap sample $\tth_{j}$ and analyze the performance of the algorithm in these settings. 

\section{Non-parametric Bootstrapping}
\label{sec:npb}
We first describe the non-parametric bootstrapping (\NPB) procedure in Section~\ref{sec:npb-proc}. We show that \NPB used in conjunction with Algorithm~\ref{algo:bb} can be provably inefficient and establish a near-linear lower bound on the regret incurred by it in the Bernoulli bandit setting (Section~\ref{sec:lb}). In Section~\ref{sec:forced}, we show that \NPB with an appropriate amount of forced exploration can result in an $O(T^{2/3})$ regret in this setting.
\subsection{Procedure}
\label{sec:npb-proc}
In order to construct the bootstrap sample $\tth_{j}$ in Algorithm~\ref{algo:bb}, we first create a new dataset $\tilde{\cD}_{j}$ by \emph{sampling with replacement}, $\vert \cD_{j} \vert$ points from $\cD_{j}$. The bootstrapped log-likelihood is equal to the log-likelihood of observing $\tilde{\cD}_{j}$. Formally, 
\begin{align}
\tilde{\cL}(\theta) = \sum_{i \in \tilde{\cD}_{j}} \log \left[ \cP( y_i | x_i, \theta ) \right] 
\label{eq:npb}
\end{align}
The bootstrap sample is computed as $\tth_{j} \in \argmax_{\theta} \tilde{\cL}(\theta)$. Observe that the sampling with replacement procedure is the source of randomness for bootstrapping and $\E[\tilde{\cD}_{j}] = \cD_{j}$. 

For the special case of Bernoulli rewards without features, a common practice is to use \emph{Laplace smoothing} where we generate positive ($1$) or negative ($0$) \emph{pseudo-examples} to be used in addition to the observed labels. Laplace smoothing is associated with two non-negative integers $\alpha_0, \beta_0$, where $\alpha_0$ (and $\beta_0$) is the \emph{pseudo-count}, equal to the number of positive (or negative) pseudo-examples. These pseudo-counts are used to ``simulate'' the prior distribution $Beta(\alpha_0, \beta_0)$. For the \NPB procedure with Bernoulli rewards, generating $\tth_{j}$ is equivalent to sampling from a Binomial distribution $Bin(n,p)$ where $n = \vert \cD_{j} \vert$ and the success probability $p$ is equal to the fraction of positive observations in $\cD_{j}$. Formally, if the number of positive observations in $\cD_{j}$ is equal to $\alpha$, then 
\begin{align}
Z \sim \mathrm{Bino} \left( n + \alpha_0 + \beta_{0}, \frac{\alpha_{0} + \alpha}{n + \alpha_0 + \beta_{0}} \right)  \quad \mathrm{and} \quad \tth_{j} = \frac{Z}{n + \alpha_0 + \beta_{0}}
\label{eq:npb-ber} 
\end{align}
\subsection{Inefficiency of Non-Parametric Bootstrapping}
\label{sec:lb}
In this subsection, we formally show that Algorithm~\ref{algo:bb} used with \NPB  might lead to an $\Omega(T^\gamma)$ regret with $\gamma$ arbitrarily close to $1$. Specifically, we consider a simple $2$-arm bandit setting, where at each round $t$, the reward of arm $1$ is independently drawn from a Bernoulli distribution with mean $\mu_1 = 1/2$, and the reward of arm $2$ is deterministic and equal to $1/4$. Furthermore, we assume that the agent knows the deterministic reward of arm $2$, but not the mean reward for arm $1$. Notice that this case is simpler than the standard two-arm Bernoulli bandit setting, in the sense that the agent also knows the reward of arm $2$. Observe that if $\tth_{1}$ is a bootstrap sample for arm $1$ (obtained according to equation~\ref{eq:npb-ber}), then the arm $1$ is selected if $\tth_{1} \geq 1/4$. Under this setting, we prove the following lower bound:
\begin{theorem}
\label{thm:lb}
If the \NPB procedure is used in the above-described 
case with pseudo-counts $(\alpha_0, \beta_0)=(1,1)$ for arm $1$, then for any $\gamma \in (0, 1)$ and any 
$T \geq \exp \left[
\frac{2}{\gamma}  \exp \left(\frac{80}{\gamma} \right)
\right ]  $,
we obtain
\[
\E[R(T)]  > \frac{T^{1-\gamma}}{32} =\Omega(T^{1-\gamma}).
\]
\end{theorem}
\begin{proof}
Please refer to Appendix~\ref{sec:lb-proof} for the detailed proof of Theorem~\ref{thm:lb}. It is proved based on a binomial tail bound (Proposition~\ref{proposition:tail_bound}) and uses the following observation: under a ``bad history", where at round $\tau$ \NPB has pulled arm $1$ for $m$ times, but all of these $m$ pulls have resulted in a reward $0$, \NPB will pull arm $1$ with probability less than $\exp \left( -m \log(m) / 20 \right)$ (Lemma~\ref{lemma:action_prob}). Hence, the number of times \NPB will pull the suboptimal arm $2$ before it pulls arm $1$ again or reach the end of the $T$ time steps follows a ``truncated geometric distribution", whose expected value is bounded in Lemma~\ref{lemma:truncated_geo}. Based on Lemma~\ref{lemma:truncated_geo}, and the fact that the probability of this bad history is $2^{-m}$, we have 
$
\E \left[ R(T) \right] \geq 2^{-(m+3)} \min \left \{
 \exp \left( m \log(m)/20 \right), T/4 
\right \}$ 
in Lemma~\ref{lemma:lb_m}. Theorem~\ref{thm:lb} is proved by setting $
m= \left \lceil \gamma \log(T)/2 \right \rceil
$. 
\end{proof}
Theorem~\ref{thm:lb} shows that, when $T$ is large enough, the \NPB procedure used in previous work~\cite{eckles2014thompson,tang2015personalized,elmachtoub2017practical} incurs an expected cumulative regret arbitrarily close to a linear regret in the order of $T$. 
It is straightforward to prove a variant of this lower bound with any constant (in terms of $T$) number of pseudo-examples.
Next, we show that \NPB with appropriate forced exploration can result in sub-linear regret. 

\subsection{Forced Exploration}
\label{sec:forced}
In this subsection, we show that \NPB, when coupled with an appropriate amount of forced exploration, can result in sub-linear regret in the Bernoulli bandit setting. In order to \emph{force} exploration, we pull each arm $m$ times before starting Algorithm~\ref{algo:bb}. The following theorem shows that for an appropriate value of $m$, this strategy can result in an $O(T^{2/3})$ upper bound on the regret. 
\begin{theorem}
\label{thm:forced-ub} In any $2$-armed bandit setting, if each arm is initially pulled $\displaystyle m = \left\lceil\left(\frac{16 \log T}{T}\right)^\frac{1}{3}\right\rceil$ times before starting Algorithm~\ref{algo:bb}, then 
\begin{align*}
  \E[R(T)] = O(T^{2/3})\,.
\end{align*}
\end{theorem}
\begin{proof}
The claim is proved in \cref{app:forced-ub} based on the following observation: If the gap of the suboptimal arm is large, the prescribed $m$ steps are sufficient to guarantee that the bootstrap sample of the optimal arm is higher than that of the suboptimal arm with a high probability at any round $t$. On the other hand, if the gap of the suboptimal arm is small, no algorithm can have high regret.
\end{proof}
Although we can remedy the \NPB procedure using this strategy, it results in a sub-optimal regret bound. In the next section, we consider a weighted bootstrapping approach as an alternative to \NPB. 
\section{Weighted Bootstrapping}
\label{sec:gen-boot}
In this section, we propose weighted bootstrapping (\GWB) as an alternative to the non-parametric bootstrap. We first describe the weighted bootstrapping procedure in Section~\ref{sec:wb-proc}. For the bandit setting with Bernoulli rewards, we show the mathematical equivalence between \GWB and TS, hence proving that \GWB attains near-optimal regret (Section~\ref{sec:WB-Bandits}).  

\subsection{Procedure}
\label{sec:wb-proc}
In order to formulate the bootstrapped log-likelihood, we use a \emph{random transformation} of the labels in the corresponding log-likelihood function. First, consider the case of Bernoulli observations where the labels $y_i \in \{0,1\}$. In this case, the log-likelihood function is given by:
\begin{align*}
\cL(\theta) = \sum_{i \in \cD_{j} } y_i \log \left( g \left( \langle \bx_i, \theta \rangle \right) \right) + (1 - y_i) \log \left( 1 - g \left( \langle \bx_i, \theta \rangle \right) \right)
\end{align*}
where the function $g(\cdot)$ is the inverse-link function. For each observation $i$, we sample a random weight $w_i$ from an exponential distribution, specifically, for all $i \in \cD_{j}$, $w_i \sim Exp(1)$. We use the following transformation of the labels: $y_i :\rightarrow w_i \cdot y_i$ and $(1 - y_i) :\rightarrow w_i \cdot (1 - y_i)$. Since we transform the labels by multiplying them with exponential weights, we refer to this case as \emph{\GWB with multiplicative exponential weights}. Observe that this transformation procedure extends the domain for the labels from values in $\{0,1\}$ to those in $\mathbb{R}$ and does not result in a valid probability mass function. However, below, we describe several advantages of using this transformation. 

Given this transformation, the bootstrapped log-likelihood function is defined as:
\begin{align}
\tilde{\cL}(\theta) = \sum_{i \in \cD_{j} } w_i \underbrace{ \left[ y_i \log \left( g \left( \langle \bx_i, \theta \rangle \right) \right) + (1 - y_i) \log \left( 1 - g \left( \langle \bx_i, \theta \rangle \right) \right) \right] }_{\ell_{i}(\theta)} = \sum_{i \in \cD_{j} } w_i \cdot l_{i}(\theta)
\label{eq:gwb}
\end{align}
Here $\ell_{i}$ is the log-likelihood of observing point $i$. As before, the bootstrap sample is computed as: $\tth_{j} \in \argmax_{\theta} \tilde{\cL}(\theta)$. Note that in \GWB, the randomness for bootstrapping is induced by the weights $w$ and that $\E_{w}[\tilde{\cL}(\theta)] = \cL(\theta)$. As a special case, in the absence of features, when $g \left( \langle \bx_i, \theta \rangle \right) = \theta$ for all $i$, assuming $\alpha_0$ positive and $\beta_0$ negative pseudo-counts and denoting $n = \vert \cD_{j} \vert$ , we obtain the following closed-form expression for computing the bootstrap sample:
\begin{align}
\tth & = \frac{ \sum_{i = 1}^{n} [ w_i \cdot y_i ] + \sum_{i = 1}^{\alpha_0} [w_i] }{ \sum_{i = 1}^{n + \alpha_0 + \beta_0} w_i } \label{eq:WB-ber}
\end{align}

Using the above transformation has the following advantages: (i) Using equation~\ref{eq:gwb}, we can interpret $\tilde{\cL}(\theta)$ as a random re-weighting (by the weights $w_i$) of the observations. This formulation is equivalent to the weighted likelihood bootstrapping procedure proposed and proven to be asymptotically consistent in the offline case in~\cite{newton1994approximate}. (ii) From an implementation perspective, computing $\tth_{j}$ involves solving a weighted maximum likelihood estimation problem. It thus has the same computational complexity as \NPB and can be solved by using black-box optimization routines. (iii) In the next section, we show that using \GWB with multiplicative exponential weights has good theoretical properties in the bandit setting. Furthermore, such a procedure of randomly transforming the labels lends itself naturally to the Gaussian case and in Appendix~\ref{app:normal}, we show that \GWB with an additive transformation using Gaussian weights is equivalent to TS. 
\subsection{Equivalence to Thompson sampling}
\label{sec:WB-Bandits}
We now analyze the theoretical performance of \GWB in the Bernoulli bandit setting. In the following proposition proved in appendix~\ref{app:bernoulli}, we show that \GWB with multiplicative exponential weights is equivalent to TS.
\begin{proposition}
If the rewards $y_i \sim Ber(\theta^{*})$, then weighted bootstrapping using the estimator in equation~\ref{eq:WB-ber} results in $\tth_{j} \sim Beta( \alpha + \alpha_{0}, \beta + \beta_0 )$, where $\alpha$ and $\beta$ is the number of positive and negative observations respectively; $\alpha_0$ and $\beta_0$ are the positive and negative pseudo-counts. In this case, \GWB is equivalent to Thompson sampling under the $Beta(\alpha_0,\beta_0)$ prior.
\label{prop:bernoulli-ts-equivalence}
\end{proposition}
Since \GWB is mathematically equivalent to TS, the bounds in~\cite{agrawal2013further} imply near-optimal regret for \GWB in the Bernoulli bandit setting.

In Appendix~\ref{app:categorical}, we show that this equivalence extends to the more general categorical (with $C$ categories) reward distribution  i.e. for $y_i \in \{1, \ldots C \}$. In appendix~\ref{app:normal}, we prove that for Gaussian rewards, \GWB with additive Gaussian weights, i.e. $w_i \sim N(0,1)$ and using the additive transformation $y_i :\rightarrow y_i + w_i$, is equivalent to TS under an uninformative $N(0,\infty)$ prior. Furthermore, this equivalence holds even in the presence of features, i.e. in the linear bandit case. Using the results in~\cite{agrawal2013thompson}, this implies that for Gaussian rewards, \GWB with additive Gaussian weights achieves near-optimal regret.   


\section{Experiments}
\label{sec:experiments}
In Section~\ref{sec:bandits-exps}, we first compare the empirical performance of bootstrapping and Thompson sampling in the bandit setting. In section~\ref{sec:cb-exps}, we describe the experimental setup for the contextual bandit setting and compare the performance of different algorithms under different feature-reward mappings. 

\subsection{Bandit setting}
\label{sec:bandits-exps}
\begin{figure*}[!ht]
\centering
        \subfigure[Bernoulli]
        {
			\includegraphics[width=0.23\textwidth]{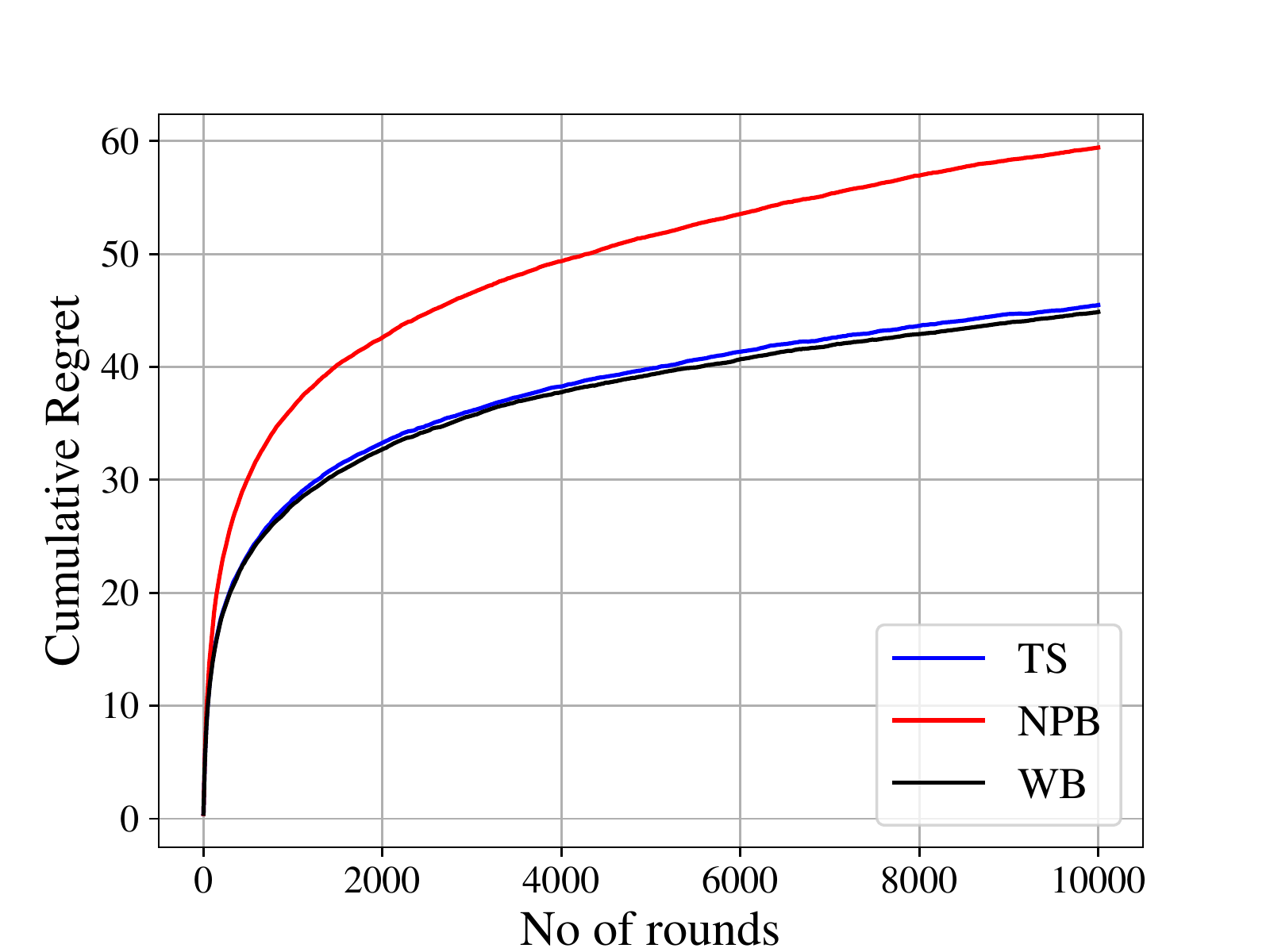}
			\label{fig:mab-bernoulli-10}
        }        
        \subfigure[Truncated-Normal]
        {
			\includegraphics[width=0.23\textwidth]{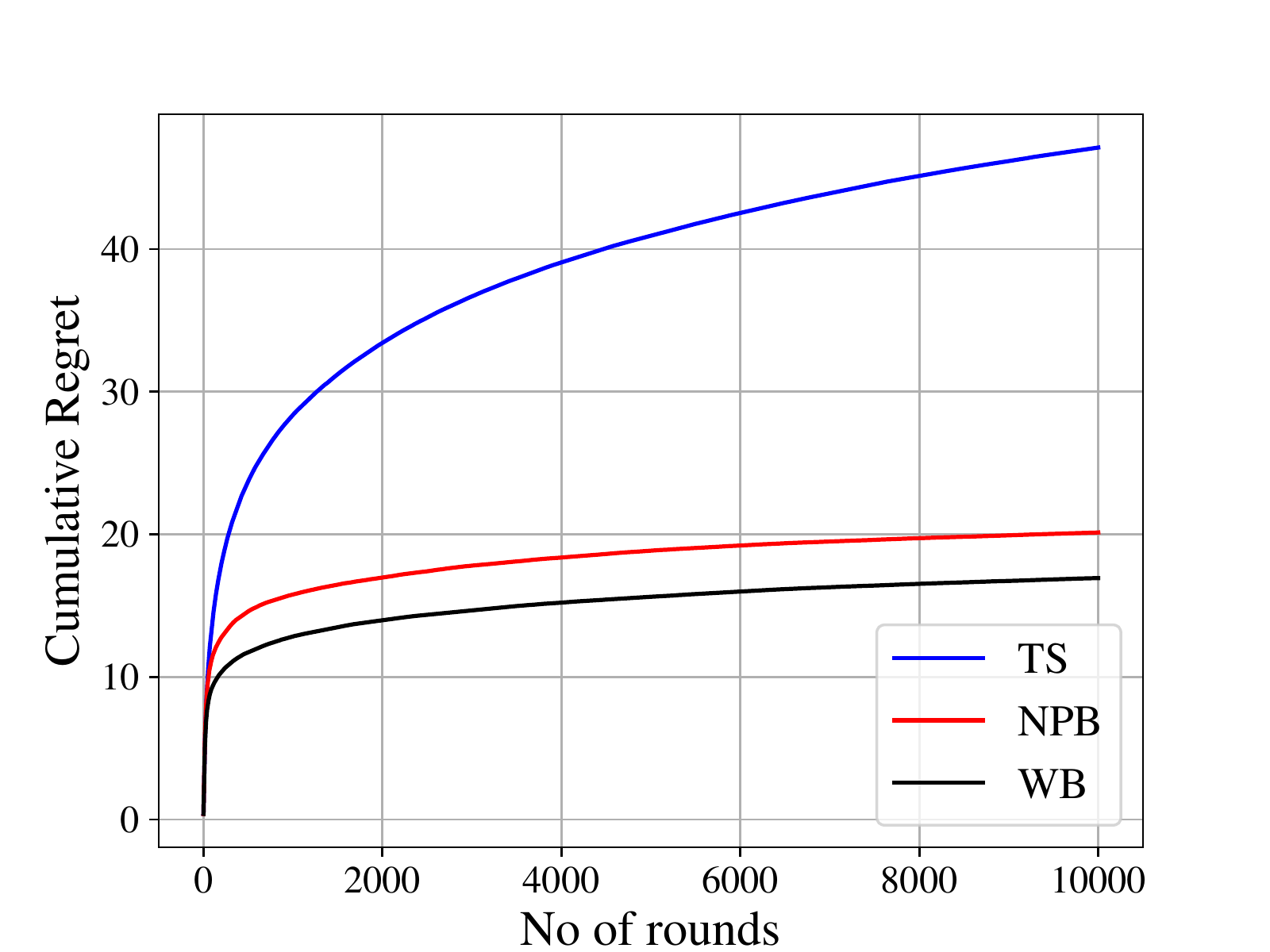}
			\label{fig:mab-truncated-normal-10}
        }              	
       \subfigure[Beta]
        {
			\includegraphics[width=0.23\textwidth]{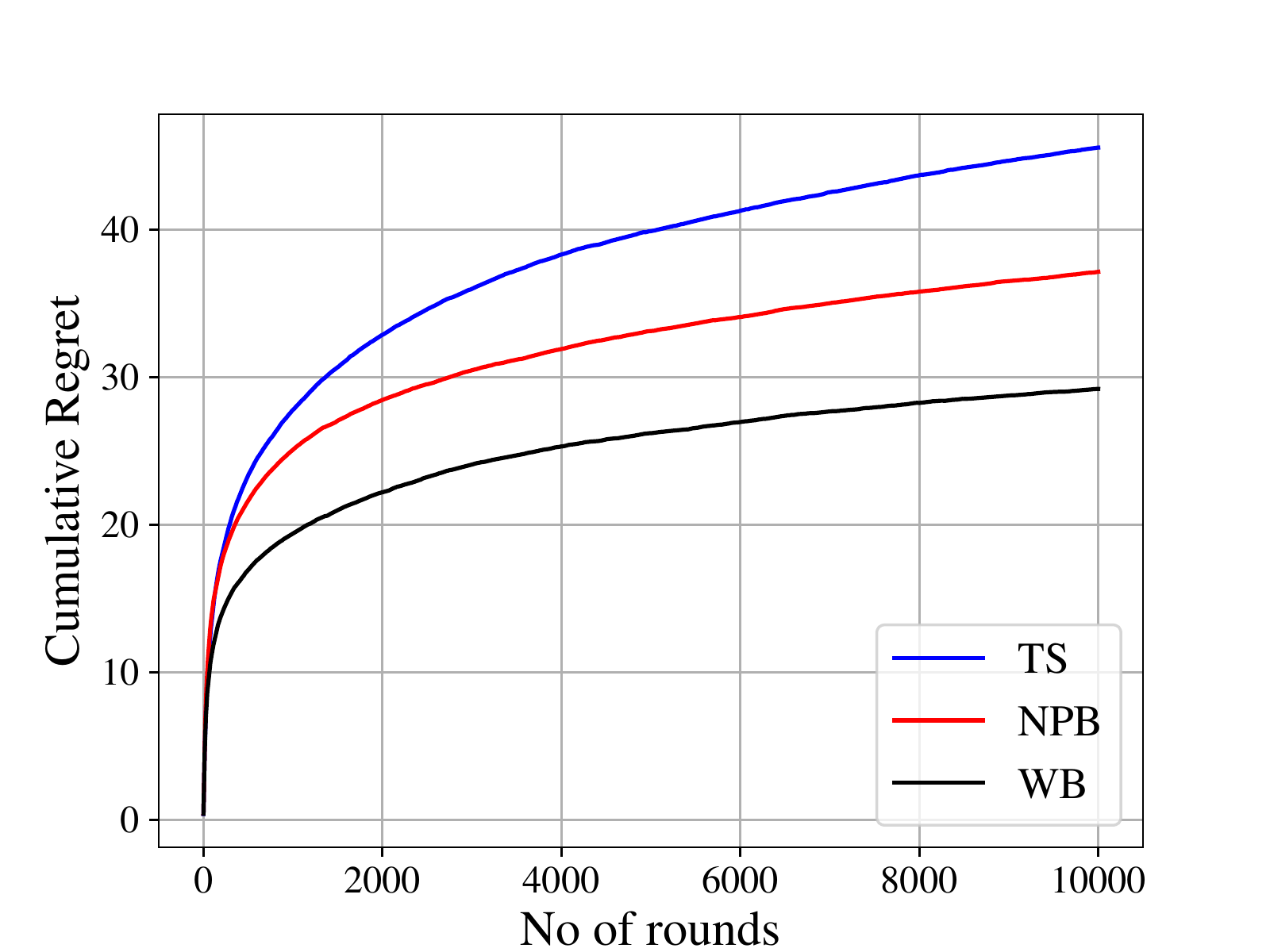}
			\label{fig:mab-beta-10}
        }              	
        \subfigure[Triangular]
        {
			\includegraphics[width=0.23\textwidth]{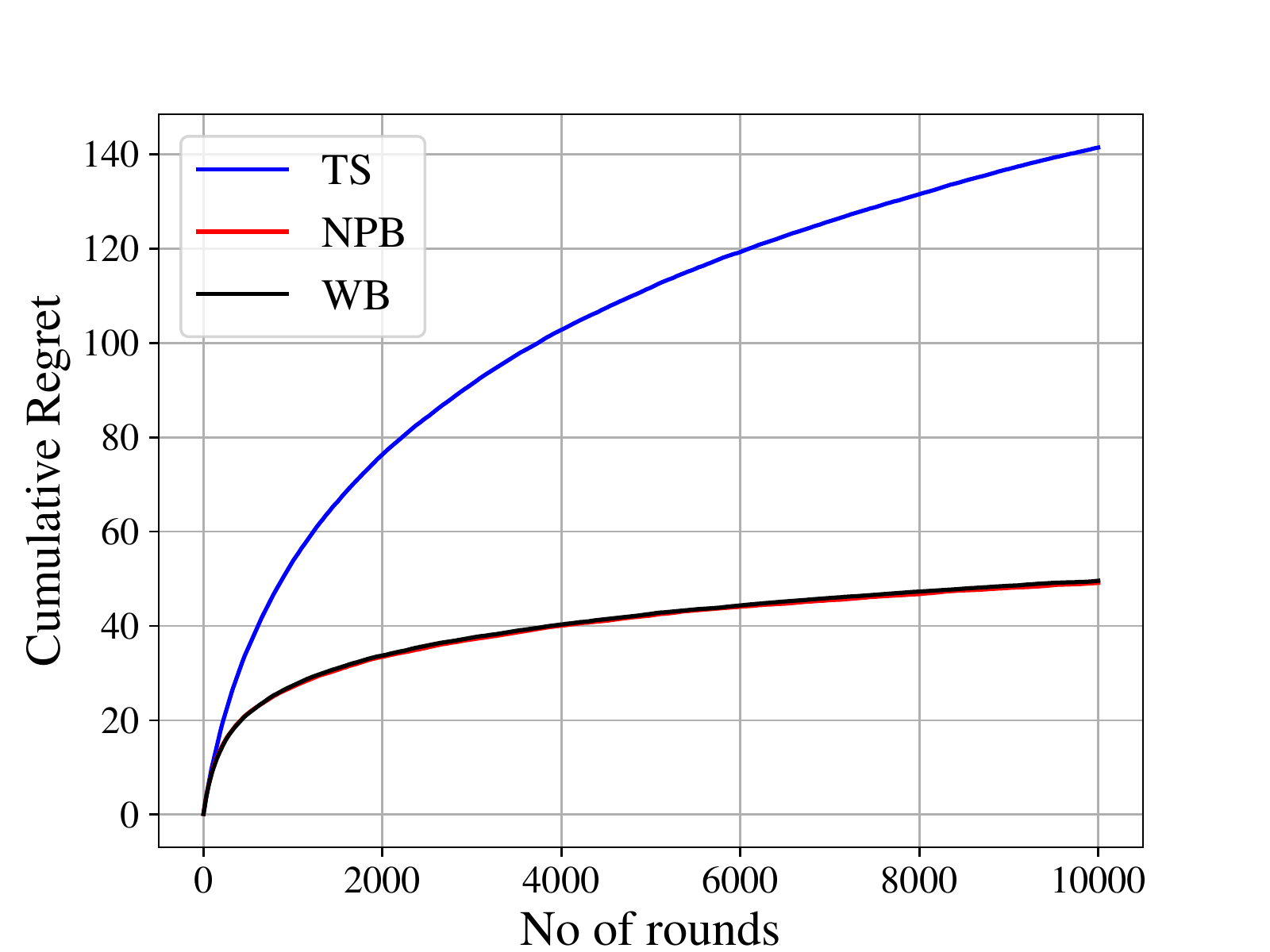}
			\label{fig:mab-triangular-10}
        }              	
        \caption{Cumulative Regret vs Number of rounds for TS, \NPB and \GWB in a bandit setting with $K = 10$ arms for (a) Bernoulli (b) Truncated-Normal (c) Beta (d) Triangular reward distributions bounded on the $[0,1]$ interval. \GWB results in the best performance in each these experiments.}
\end{figure*}    
We consider $K = 10$ arms (refer to Appendix~\ref{app:additional-results} for results with other values of $K$), a horizon of $T = 10^{4}$ rounds and average our results across $10^{3}$ runs. We perform experiments for four different reward distributions - Bernoulli, Truncated Normal, Beta and the Triangular distribution, all bounded on the $[0,1]$ interval. In each run and for each arm $j$, we choose the expected reward $\mu_j$ (mean of the corresponding distribution) to be a uniformly distributed random number in $[0,1]$. For the Truncated-Normal distribution, we choose the standard deviation to be equal to $10^{-4}$, whereas for the Beta distribution, the shape parameters of arm $j$ are chosen to be $\alpha = \mu_{j}$ and $\beta = 1 - \mu_{j}$. We use the $Beta(1,1)$ prior for TS. In order to use TS on distributions other than Bernoulli, we follow the procedure proposed in~\cite{agrawal2013further}: for a reward in $[0,1]$ we flip a coin with the probability of obtaining $1$ equal to the reward, resulting in a binary ``pseudo-reward''. This pseudo-reward is then used to update the Beta posterior as in the Bernoulli case. For \NPB and \GWB, we use the estimators in equations~\ref{eq:npb-ber} and~\ref{eq:WB-ber} respectively. For both of these, we use the pseudo-counts $\alpha_0 = \beta_0 = 1$.

In the Bernoulli case, \NPB obtains a higher regret as compared to both TS and \GWB which are equivalent. For the other distributions, we observe that both \GWB and \NPB (with \GWB resulting in consistently better performance) obtain lower cumulative regret than the modified TS procedure. This shows that for distributions that do not admit a conjugate prior, \GWB (and \NPB) can be directly used and results in good empirical performance as compared to making modifications to the TS procedure. 

\subsection{Contextual bandit setting}
\label{sec:cb-exps}
\begin{figure*}[!ht]
\centering
        \subfigure[CovType]
        {
			\includegraphics[width=0.48\textwidth]{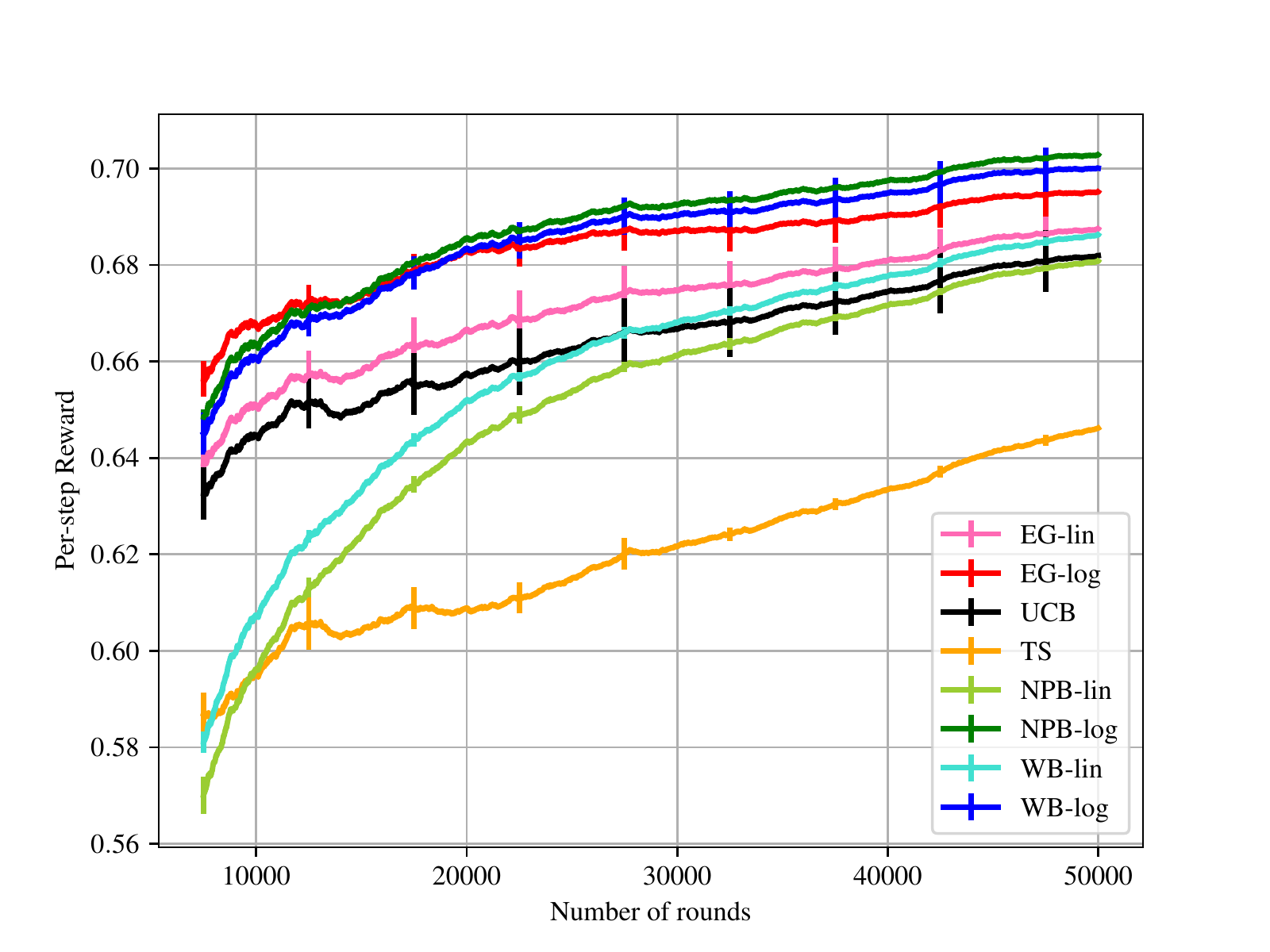}
			\label{fig:cb-covtype}
        }        
        \subfigure[MNIST]
        {
			\includegraphics[width=0.48\textwidth]{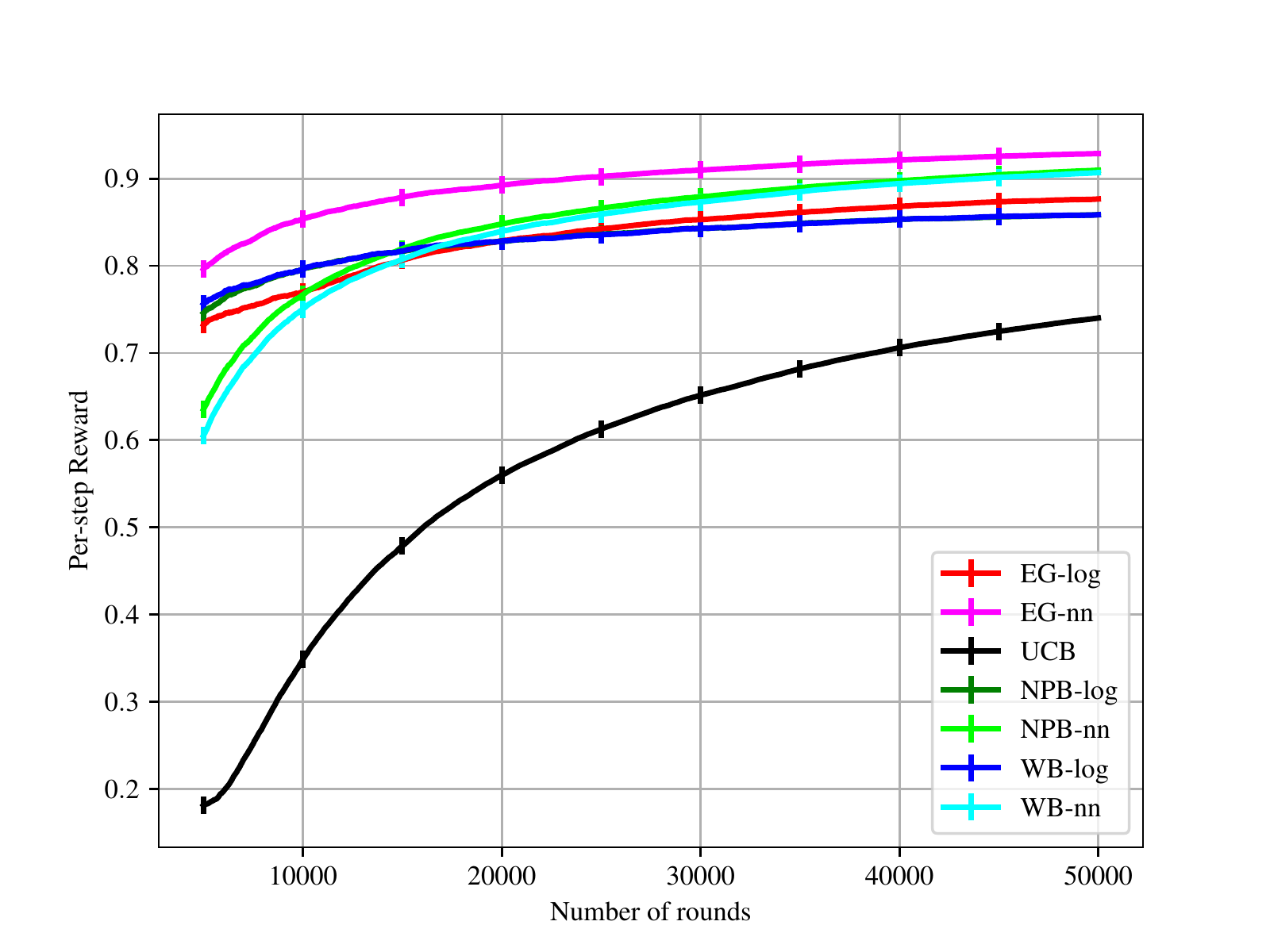}
			\label{fig:cb-mnist}
        }        
\caption{Expected per-step reward vs Number of Rounds for (a) CovType (b) MNIST datasets. Using non-linear models leads to higher per-step reward compared to linear bandit approaches. \NPB and \GWB with non-linear models perform well and match the performance of a tuned EG approach.}
\end{figure*}

We adopt the one-versus-all multi-class classification setting for evaluating contextual bandits \cite{agarwal2014taming,elmachtoub2017practical}. Each arm corresponds to a class. In each round, the algorithm receives a reward of one if the context vector belongs to the class corresponding to the selected arm and zero otherwise. Each arm maintains an independent set of sufficient statistics that map the context vector to the observed binary reward. We use two multi-class datasets: CoverType ($K = 7$ and $d = 54$) and MNIST ($K = 10$ and $d = 784$). The number of rounds in experiments is $T = 50000$ and we average results over $5$ independent runs. We experiment with LinUCB~\cite{abbasi2011improved}, which we call UCB, linear Thompson sampling (TS)~\cite{agrawal2013thompson}, $\epsilon$-greedy (EG)~\cite{langford2008epoch}, non-parametric bootstrapping (\NPB), and weighted bootstrapping (\GWB). For EG, \NPB and \GWB, we consider three model classes: linear regression (suffix ``-lin'' in plots), logistic regression (suffix ``-log'' in plots), and a single hidden-layer (with $100$ hidden neurons) fully-connected neural network (suffix ``-nn'' in plots). Since we compare various bandit algorithms and model classes, we use the expected per-step reward, $(\sum_{t = 1}^{T} \E[r_t] )/T$, as our performance metric. 

For EG, we experimented extensively with many different exploration schedules. We found that $\epsilon_t = 50 / (50 + t)$ leads to the best performance on both of our datasets. In practice, it is not possible to do such tuning on a new problem. Therefore, the EG results in this paper should be viewed as a proxy for the ``best'' attainable performance. As alluded to in Section~\ref{sec:introduction}, we implement bootstrapping using stochastic optimization with warm-start, in contrast to approximating it as in~\cite{elmachtoub2017practical, tang2015personalized}. Specifically, we use stochastic gradient descent to compute the MLE for the bootstrapped log-likelihood and warm-start the optimization at round $t$ by the solution from the previous round $t-1$. For linear and logistic regression, we optimize until we reach an error threshold of $10^{-3}$. For the neural network, we take $1$ pass over the dataset in each round. To ensure that the results do not depend on our specific choice of optimization, we use scikit-learn~\cite{sklearn_api} with stochastic optimization, and default optimization options for both linear and logistic regression. For the neural network, we use the Keras library~\cite{chollet2015} with the ReLU non-linearity for the hidden layer and sigmoid in the output layer, along with SGD and its default configuration. Preliminary experiments suggested that our procedure leads to better runtime as compared to~\cite{elmachtoub2017practical} and better performance than the approximation proposed in~\cite{tang2015personalized}, while also alleviating the need to tune any hyper-parameters. 

In the prior work on bootstrapping \cite{elmachtoub2017practical,tang2015personalized} for contextual bandits, the algorithm was initialized through forced exploration, where each arm is explored $m$ times at the beginning; or equivalently assigned $m$ pseudo-examples that are randomly sampled context vectors. Such a procedure introduces yet another tunable parameter $m$. Therefore, we propose the following parameter-free procedure. Let $v_1, \dots, v_d$ be the eigenvectors of the covariance matrix of the context vectors, and $\lambda_1^2, \dots, \lambda_d^2$ be the corresponding eigenvalues. For each arm, we add $4 d$ pseudo-examples: for all $i \in [d]$, we include the vectors $\lambda_i v_i$ and $- \lambda_i v_i$ each with both $0$ and $1$ labels. Since $\lambda_i$ is the standard deviation of features in the direction of $v_i$, this procedure ensures that we maintain enough variance in the directions where the contexts lie. In the absence of any prior information about the contexts, we recommend using $O(d)$ samples from an isotropic multivariate Gaussian and validate that it led to comparable performance on the two datasets. 

We plot the expected per-step reward of all compared methods on the CoverType and MNIST datasets in figures~\ref{fig:cb-covtype} and~\ref{fig:cb-mnist}, respectively. In figure~\ref{fig:cb-covtype}, we observe that EG, \NPB, and \GWB with logistic regression have the best performance in all rounds. The linear methods (EG, UCB, and bootstrapping) perform similarly, and slightly worse than logistic regression whereas TS has the worst performance. Neural networks perform similarly to logistic regression and we do not plot them here. This experiment shows that even for a relatively simple dataset, like CovType, a more expressive non-linear model can lead to better performance. This effect is more pronounced in figure~\ref{fig:cb-mnist}. For this dataset, we only show the best performing linear method, UCB. The performance of other linear methods, including those with bootstrapping, is comparable to or worse than UCB. We observe that non-linear models yield a much higher per-step reward, with the neural network performing the best. For both logistic regression and neural networks, the performance of both bootstrapping methods is similar and only slightly worse, respectively, than that of a tuned EG method. Both \NPB and \GWB are computationally efficient; on the CovType dataset, \NPB and \GWB with logistic regression take on average, $0.15$ and $0.16$ seconds per round, respectively. On the MNIST dataset, \NPB and \GWB have an average runtime of $1.16$ and $1.88$ seconds per round, respectively, when using logistic regression; and $1.16$ and $1.69$ seconds per round, respectively, when using a neural network.





\section{Discussion}
\label{sec:discussion}
We showed that the commonly used non-parametric bootstrapping procedure can be provably inefficient. As an alternative, we proposed the weighted bootstrapping procedure, special cases of which become equivalent to TS for common reward distributions such as Bernoulli and Gaussian. On the empirical side, we showed that the \GWB procedure has better performance than a modified TS scheme for several bounded distributions in the bandit setting. In the contextual bandit setting, we provided guidelines to make bootstrapping simple and efficient to implement and showed that non-linear versions of bootstrapping have good empirical performance. Our work raises several open questions: does bootstrapping result in near-optimal regret for generalized linear models? Under what assumptions or modifications can \NPB be shown to have good performance? On the empirical side, evaluating bootstrapping across multiple datasets and comparing it against TS with approximate sampling is an important future direction.

\newpage

\appendix
\section{Proof for Theorem~\ref{thm:lb}}
\label{sec:lb-proof}
We prove Theorem~\ref{thm:lb} in this section. First, we have the following tail bound for Binomial random variables:
\begin{proposition}[Binomial Tail Bound]
\label{proposition:tail_bound}
Assume that random variable $X \sim \mathrm{Bino}(n,p)$, then for any $k$ s.t. $np <k <n$, we have
\[
\cP(X \geq k) \leq \exp \left( -n D \left( \frac{k}{n} \middle \| p \right )\right ),
\]
where $D \left( \frac{k}{n} \middle \| p \right )$ is the KL-divergence between $\frac{k}{n} $ and $p$.
\end{proposition}
Please refer to \cite{Arratia1989} for the proof of Proposition~\ref{proposition:tail_bound}.

Notice that for our considered case, the ``observation history" of the agent at the beginning of time $t$ is completely characterized by a 
triple $\cH_t=(\alpha_{t-1}, \beta_{t-1}, t)$, where $\alpha_{t-1}$ is the number of times arm $1$ has been pulled from time $1$ to $t-1$ and the realized reward is $1$, plus the pseudo count $\alpha_0$; similarly, $\beta_{t-1}$ is the number of times arm $1$ has been pulled from time $1$ to $t-1$ and the realized reward is $0$, plus the pseudo count $\beta_0$.
Moreover, conditioning on this history $\cH_t$, the probability that the agent will pull arm $1$ under the \NPB only depends on $(\alpha_{t-1}, \beta_{t-1})$. To simplify the exposition, we use $P(\alpha_{t-1}, \beta_{t-1})$
to denote this conditional probability. The following lemma bounds this probability in a ``bad" history: 
\begin{lemma}
\label{lemma:action_prob}
Consider a ``bad" history $\cH_{t}$ with $\alpha_{t-1}=1$ and $\beta_{t-1}=1+m$ for some integer $m \geq 15$, then we have
\[
P(\alpha_{t-1}, \beta_{t-1}) < \exp \left( -(m+2) \log(m+2) / 20 \right) < \exp \left( -m \log(m) / 20 \right).
\]
\end{lemma}
\begin{proof}
Recall that by definition, we have
\begin{align}
P(\alpha_{t-1}, \beta_{t-1}) =& \,  P(1, m+1) \nonumber \\
\stackrel{(a)}{=} & \, \cP \left( w_t \geq 1/4 \, \middle | \, \alpha_{t-1}=1, \beta_{t-1} =m+1 \right) \nonumber \\
=& \, \cP \left( (m+2) w_t \geq (m+2)/4 \, \middle | \, \alpha_{t-1}=1, \beta_{t-1} =m+1 \right) \nonumber \\
\stackrel{(b)}{\leq} & \, \exp \left( -(m+2) D \left(\frac{1}{4} \middle \| \frac{1}{m+2} \right) \right),
\end{align}
where (a) follows from the \NPB procedure in this case, and (b) follows from Proposition~\ref{proposition:tail_bound}.
Specifically, recall that $(m+2) w_t \sim \mathrm{Bino}(m+2, 1/(m+2))$, and
$
(m+2)/4 > (m+2) \frac{1}{m+2}=1
$
for $m \geq 15$. Thus, the conditions of Proposition~\ref{proposition:tail_bound} hold in this case. Furthermore, we have
\begin{align}
 D \left(\frac{1}{4} \middle \| \frac{1}{m+2} \right) =& \, \frac{1}{4} \log \left( \frac{m+2}{4}\right) + \frac{3}{4} \log \left( \frac{3(m+2)}{4(m+1)} \right) \nonumber \\
 \geq & \, \frac{1}{4} \log(m+2) - \frac{1}{4} \log(4) +  \frac{3}{4} \log \left( \frac{3}{4} \right) \nonumber \\
 = & \, \frac{1}{20} \log(m+2) + \left [
 \frac{1}{5} \log(m+2) - \frac{1}{4} \log(4) +  \frac{3}{4} \log \left( \frac{3}{4} \right) 
 \right] \nonumber \\
 \stackrel{(c)}{>} & \, \frac{1}{20} \log(m+2) ,
\end{align}
where (c) follows from the fact that $ \frac{1}{5} \log(m+2) - \frac{1}{4} \log(4) +  \frac{3}{4} \log \left( \frac{3}{4} \right) \geq 0$ for $m > 15$. Thus we have
\begin{align}
P(\alpha_{t-1}, \beta_{t-1}) \leq & \, \exp \left( -(m+2) D \left(\frac{1}{4} \middle \| \frac{1}{m+2} \right) \right) \nonumber \\
< & \, \exp \left( -(m+2) \log(m+2)/20 \right) \nonumber \\
< & \, \exp \left( -m \log(m)/20 \right).
\end{align}
\end{proof}

The following technical lemma derives the expected value of a truncated geometric random variable, 
as well as a lower bound on it, which will be used in the subsequent analysis:
\begin{lemma}[Expected Value of Truncated Geometric R.V.]
\label{lemma:truncated_geo}
Assume that $Z$ is a truncated geometric r.v. with parameter $p \in (0, 1)$ and integer $l \geq 1$. Specifically, the domain of 
$Z$ is $\{ 0,1, \ldots, l \}$, and $\cP(Z=i)=(1-p)^i p$ for $i=0, 1, \ldots, l-1$ and $\cP(Z=l)=(1-p)^l $. Then we have
\[
\E(Z)= \left[\frac{1}{p}-1 \right] \left[1 -(1-p)^l \right] \geq \frac{1}{2} \min \left \{ \frac{1}{p}-1 , \, l (1-p) \right \} .
\]
\end{lemma}
\begin{proof}
Notice that by definition, we have
\[
\E(Z)= p \underbrace{ \left[ \sum_{i=0}^{l-1} i (1-p)^i \right] }_{A} +l (1-p)^l  
\]
Define the shorthand notation $A=\sum_{i=0}^{l-1} i (1-p)^i $, we have
\begin{align}
(1-p)A=& \, \sum_{i=0}^{l-1} i (1-p)^{i+1} =\sum_{i=1}^{l} (i-1) (1-p)^{i} \nonumber \\
 =& \,  \sum_{i=0}^{l} i (1-p)^{i} -\left[\frac{1}{p}-1 \right] \left[1 -(1-p)^l \right] \nonumber \\
 =& \, A + l(1-p)^{l} -\left[\frac{1}{p}-1 \right] \left[1 -(1-p)^l \right].
\end{align}
Recall that $\E(Z)= pA + l (1-p)^l  $, we have proved that 
$\E(Z)= \left[\frac{1}{p}-1 \right] \left[1 -(1-p)^l \right] $.

Now we prove the lower bound. First, we prove that 
\begin{equation}
\label{eqn:truncated_grv}
(1-p)^l \leq 1- \frac{pl}{1+pl}
\end{equation}
always holds by induction on $l$. Notice that when $l=1$, the LHS of equation~(\ref{eqn:truncated_grv}) is $1-p$, and the RHS of equation~(\ref{eqn:truncated_grv})
is $\frac{1}{1+p}$. Hence, this inequality trivially holds in the base case. Now assume that equation~(\ref{eqn:truncated_grv}) holds for $l$, we prove that it also holds for $l+1$. Notice that
\begin{align}
(1-p)^{l+1} =& \, (1-p)^{l} (1-p) \stackrel{(a)}{\leq} \left( 1- \frac{pl}{1+pl} \right) (1-p)  \nonumber \\
=& \, 1 - \frac{p(l+1)}{1+pl} +\frac{p}{1+pl}-p +\frac{p^2 l}{1+pl} \nonumber \\
=& \, 1 - \frac{p(l+1)}{1+pl} \leq 1 - \frac{p(l+1)}{1+p(l+1)} \, ,
\end{align}
where (a) follows from the induction hypothesis. Thus equation~(\ref{eqn:truncated_grv}) holds for all $p$ and $l$.
Notice that equation~\ref{eqn:truncated_grv} implies that
\[
\E(Z)= \left[\frac{1}{p}-1 \right] \left[1 -(1-p)^l \right] \geq \left[\frac{1}{p}-1 \right] \frac{pl}{1+pl} \, .
\]
We now prove the lower bound. Notice that for any $l$, $ \frac{pl}{1+pl}$ is an increasing function of $p$, thus for $p \geq 1/l$, we have
\[
 \left[\frac{1}{p}-1 \right] \frac{pl}{1+pl} \geq \left[\frac{1}{p}-1 \right]/2 \geq \frac{1}{2} \min \left \{ \frac{1}{p}-1 , \, l (1-p) \right \}.
\]
On the other hand, if $p \leq 1/l$, we have
\[
 \left[\frac{1}{p}-1 \right] \frac{pl}{1+pl}  = \frac{(1-p)l}{1+pl} \geq (1-p)l/2  \geq \frac{1}{2} \min \left \{ \frac{1}{p}-1 , \, l (1-p) \right \}\, .
\]
Combining the above results, we have proved the lower bound on $\E (Z)$.
\end{proof}

We then prove the following lemma:
\begin{lemma}[Regret Bound Based on $m$]
\label{lemma:lb_m}
When \NPB is applied in the considered case, for any integer $m$ and time horizon $T$ satisfying
$15 \leq m \leq T$, we have
\[
\E \left[ R(T) \right] \geq \frac{2^{-m}}{8} \min \left \{
 \exp \left( \frac{m \log(m)}{20} \right), \frac{T}{4} 
\right \}\, .
\]
\end{lemma}
\begin{proof}
We start by defining the bad event $\cG$ as
\[
\cG = \left \{ \exists t =1,2,\ldots \text{ s.t. } \alpha_{t-1}=1 \text{ and } \beta_{t-1}=m+1 \right \} \, .
\]
Thus, we have $\E \left[ R(T) \right] \geq \cP \left( \cG \right ) \E \left[ R(T) \middle | \, \cG \right]$. 
Since $\alpha_t \geq 1$ for all $t=1,2,\ldots$, with probability $1$, the agent will pull arm $1$ infinitely often. Moreover, the event $\cG$
only depends on the outcomes of the first $m$ pulls of arm $1$. Thus we have
$\cP \left( \cG \right ) = 2^{-m}$. Furthermore, conditioning on $\cG$, we define the stopping time $\tau$ as
\[
\tau = \min \left \{  t =1,2,\ldots \text{ s.t. } \alpha_{t-1}=1 \text{ and } \beta_{t-1}=m+1 \right \} \, .
\]
Then we have
\begin{align}
\E \left[ R(T) \right]  \geq & \, \cP \left( \cG \right ) \E \left[ R(T) \middle | \, \cG \right] = 2^{-m} \E \left[ R(T) \middle | \, \cG \right] \nonumber \\
= & \,  2^{-m} \left [
\cP \left (\tau > T/2 \middle | \cG \right) \E \left[ R(T) \middle | \, \cG,  \tau > T/2 \right]
+
\cP \left (\tau \leq T/2 \middle | \cG \right) \E \left[ R(T) \middle | \, \cG,  \tau \leq T/2 \right]
\right ] \nonumber \\
\geq & \,  2^{-m} 
\min \left \{
\E \left[ R(T) \middle | \, \cG,  \tau > T/2 \right], \, 
\E \left[ R(T) \middle | \, \cG,  \tau \leq T/2 \right]
\right \}
\end{align}
Notice that conditioning on event $\{ \cG,  \tau > T/2 \}$, in the first $\lfloor T/2 \rfloor$ steps, the agent either pulls arm $2$ or pulls arm $1$ but receives a reward
$0$, thus, by definition of $R(T)$, we have
\[
\E \left[ R(T) \middle | \, \cG,  \tau > T/2 \right] \geq \frac{ \lfloor T/2 \rfloor }{4}.
\]

On the other hand, if $\tau \leq T/2$, notice that for any time $t \geq \tau$ with history $\cH_{t}=(\alpha_{t-1}, \beta_{t-1}, t)$ s.t. 
$(\alpha_{t-1}, \beta_{t-1}) = (1, m+1)$, the agent will pull arm $1$
conditionally independently with probability $P(1, m+1)$. Thus, conditioning on $\cH_{\tau}$, the number of times the agent will pull arm $2$ before it pulls arm
$1$ again follows the truncated geometric distribution with parameter $P(1, m+1)$ and $T-\tau+1$. From Lemma~\ref{lemma:truncated_geo}, for any $\tau \leq T/2$, we have
\begin{align}
\E \left[ R(T) \middle | \, \cG,  \tau  \right] \stackrel{(a)}{\geq} & \, \frac{1}{8}  \min \left \{ \frac{1}{P(1, m+1)}-1 , \, (T-\tau+1)(1-P(1, m+1)) \right \}  \nonumber \\
\stackrel{(b)}{\geq}  & \,  \frac{1}{8}  \min \left \{ \frac{1}{P(1, m+1)}-1 , \, \frac{T}{2}(1-P(1, m+1)) \right \}  \nonumber \\
\stackrel{(c)}{>}  & \,  \frac{1}{8}  \min \left \{ \exp \left( (m+2) \log(m+2) / 20 \right)-1 , \, \frac{T}{4} \right \}  \nonumber \\
\stackrel{(d)}{\geq}  & \,  \frac{1}{8}  \min \left \{ \exp \left( m \log(m) / 20 \right) , \, \frac{T}{4} \right \} \, ,
\end{align}
notice that a factor of $1/4$ in inequality (a) is due to the reward gap. Inequality (b) follows from the fact that $\tau \leq T/2$; 
inequality (c) follows from Lemma~\ref{lemma:action_prob}, which states that for $m \geq 15$, we have
$ P(\alpha_{t-1}, \beta_{t-1}) < \exp \left( -(m+2) \log(m+2) / 20 \right) < \frac{1}{2}$;
inequality (d) follows from the fact that for $m \geq 15$, we have
\[
 \exp \left( (m+2) \log(m+2) / 20 \right)-1 > \exp \left( m \log(m) / 20 \right) \, .
\]
Finally, notice that
\[
\E \left[ R(T) \middle | \, \cG,  \tau \leq T/2 \right ] = \sum_{\tau \leq T/2} \cP(\tau |  \cG,  \tau \leq T/2 ) \E \left[ R(T) \middle | \, \cG,  \tau  \right] > 
\frac{1}{8}  \min \left \{ \exp \left( \frac{m \log(m)}{20} \right) , \, \frac{T}{4} \right \} \, .
\]
Thus, combining everything together, we have
\begin{align}
\E[R(T)]  \geq & \, 2^{-m} 
\min \left \{
\E \left[ R(T) \middle | \, \cG,  \tau > T/2 \right], \, 
\E \left[ R(T) \middle | \, \cG,  \tau \leq T/2 \right]
\right \} \nonumber \\
>& \, \frac{2^{-m}}{4} \min \left \{
\frac{1}{2} \exp \left( \frac{m \log(m)}{20} \right), \frac{T}{8} , \lfloor \frac{T}{2} \rfloor
\right \} \nonumber \\
= & \, \frac{2^{-m}}{4} \min \left \{
\frac{1}{2} \exp \left( \frac{m \log(m)}{20} \right), \frac{T}{8} 
\right \} \, ,
\end{align}
where the last equality follows from the fact that $ \frac{T}{8} < \lfloor \frac{T}{2} \rfloor$ for $T \geq 15$.
This concludes the proof.
\end{proof}

Finally, we prove Theorem~\ref{thm:lb}.
\begin{proof}
For any given $\gamma \in (0,1)$, we choose 
$
m= \left \lceil \frac{\gamma \log(T)}{2} \right \rceil
$. 
Since
\[ T \geq \exp \left[
\frac{2}{\gamma}  \exp \left(\frac{80}{\gamma} \right)
\right ]  \,
, \] 
we have
\[
T \gg m= \left \lceil \frac{\gamma \log(T)}{2} \right \rceil \geq \exp \left(\frac{80}{\gamma} \right) \geq \exp \left(80 \right) \gg 15,
\]
thus, Lemma~\ref{lemma:lb_m} is applicable. Notice that
\[
\E \left[ R(T) \right] \geq \frac{2^{-m}}{8} \min \left \{
 \exp \left( \frac{m \log(m)}{20} \right), \frac{T}{4} 
\right \} >  
\frac{\exp(-m)}{8} \min \left \{
 \exp \left( \frac{m \log(m)}{20} \right), \frac{T}{4} 
\right \} \, .
\]
Furthermore, we have
\[
\exp(-m) T > \exp \left(
-\gamma \log(T)
\right) T = T^{1-\gamma}\, ,
\]
where the first inequality follows from $m= \left \lceil \frac{\gamma \log(T)}{2} \right \rceil < \gamma \log(T)$. On the other hand, we have
\[
\exp(-m)  \exp \left( \frac{m \log(m)}{20} \right) = \exp \left (
 \frac{m \log(m)}{20} -m
\right) \geq 
 \exp \left (
 \frac{m \log(m)}{40} 
\right) \, ,
\]
where the last inequality follows from the fact that $ \frac{m \log(m)}{40}  \geq m$, since
$m \geq \exp(80)$. Notice that we have
\[
 \exp \left (
 \frac{m \log(m)}{40} 
\right) \geq 
 \exp \left (
 \frac{\gamma \log(T) \log(\frac{\gamma \log(T)}{2})}{80} 
\right) \geq T \, ,
\]
where the first inequality follows from the fact that $m \geq \frac{\gamma \log(T)}{2}$, and the second inequality follows from 
$T \geq \exp \left[
\frac{2}{\gamma}  \exp \left(\frac{80}{\gamma} \right)
\right ] $.
Putting it together, we have
\[
\E \left[ R(T) \right]  > 
\frac{1}{8} \min \left \{
T, \frac{T^{1-\gamma}}{4} 
\right \} = \frac{T^{1-\gamma}}{32} \, .
\]
This concludes the proof for Theorem~\ref{thm:lb}.
\end{proof}


\newcommand{\etal}{\emph{et al.}}
\section{Proof for Theorem 2}
\label{app:forced-ub}

For simplicity of exposition, we consider $2$ arms with means $\mu_1 > \mu_2$. Let $\Delta = \mu_1 - \mu_2$. Let $\bar{\mu}_t(k)$ be the mean of the history of arm $k$ at time $t$ and $\hat{\mu}_t(k)$ be the mean of the bootstrap sample of arm $k$ at time $t$. Note that both are random variables. Each arm is initially explored $m$ times. Since $\bar{\mu}_t(k)$ and $\hat{\mu}_t(k)$ are estimated from random samples of size at least $m$, we get from Hoeffding's inequality (Theorem 2.8 in Boucheron \etal~\cite{boucheron2013concentration}) that
\begin{align*}
  P(\bar{\mu}_t(1) \leq \mu_1 - \epsilon) & \leq \exp[-2 \epsilon^2 m]\,, \\
  P(\bar{\mu}_t(2) \geq \mu_2 + \epsilon) & \leq \exp[-2 \epsilon^2 m]\,, \\
  P(\hat{\mu}_t(1) \leq \bar{\mu}_t(1) - \epsilon) & \leq \exp[-2 \epsilon^2 m]\,, \\
  P(\hat{\mu}_t(2) \geq \bar{\mu}_t(2) + \epsilon) & \leq \exp[-2 \epsilon^2 m]
\end{align*} 
for any $\epsilon > 0$ and time $t > 2 m$. The first two inequalities hold for any $\mu_1$ and $\mu_2$. The last two hold for any $\bar{\mu}_t(1)$ and $\bar{\mu}_t(2)$, and therefore also in expectation over their random realizations. Let $\mathcal{E}$ be the event that the above inequalities hold jointly at all times $t > 2 m$ and $\bar{\mathcal{E}}$ be the complement of event $\mathcal{E}$. Then by the union bound,
\begin{align*}
  P(\bar{\mathcal{E}}) \leq 4 T \exp[-2 \epsilon^2 m]\,.
\end{align*} 
By the design of the algorithm, the expected $T$-step regret is bounded from above as
\begin{align*}
  \E[R(T)]
  & = \Delta m + \Delta \sum_{t = 2 m + 1}^T \E[\I\{j_t = 2\}] \\
  & \leq \Delta m + \Delta \sum_{t = 2 m + 1}^T \E[\I\{j_t = 2, \ \mathcal{E}\}] + 4 T^2 \exp[-2 \epsilon^2 m]\,,
\end{align*} 
where the last inequality follows from the definition of event $\mathcal{E}$ and observation that the maximum $T$-step regret is $T$. Let
\begin{align*}
  \quad m = \left\lceil\frac{16}{\tilde{\Delta}^2} \log T\right\rceil\,, \quad \epsilon = \frac{\tilde{\Delta}}{4}\,,
\end{align*} 
where $\tilde{\Delta}$ is a tunable parameter that determines the number of exploration steps per arm. From the definition of $m$ and $\tilde{\Delta}$, and the fact that $\E[\I\{j_t = 2, \ \mathcal{E}\}] = 0$ when $\tilde{\Delta} \leq \Delta$, we have that
\begin{align*}
  \E[R(T)] \leq
  \frac{16 \Delta}{\tilde{\Delta}^2} \log T + \tilde{\Delta} T + \Delta + 4\,.
\end{align*} 
Finally, note that $\Delta \leq 1$ and we choose $\displaystyle \tilde{\Delta} = \left(\frac{16 \log T}{T}\right)^\frac{1}{3}$ that optimizes the upper bound.

\section{Weighted bootstrapping and equivalence to TS}
\label{app:ts-equivalence}
In this section, we prove that for the common reward distributions, \GWB becomes equivalent to TS for specific choices of the weight distribution and the transformation function. 
\subsection{Using multiplicative exponential weights}
In this subsection, we consider multiplicative exponential weights, implying that $w_i \sim Exp(1)$ and $\cT(y_i,w_i) = y_i \cdot w_i$. We show that in this setting \GWB is mathematically equivalent to TS for Bernoulli and more generally categorical rewards. 
\subsubsection{Proof for Proposition~\ref{prop:bernoulli-ts-equivalence}}
\label{app:bernoulli}
\begin{proof}
Recall that the bootstrap sample is given as:
\begin{align*}
\tth & = \frac{ \sum_{i = 1}^{n} [ w_i \cdot y_i ] + \sum_{i = 1}^{\alpha_0} [w_i] }{ \sum_{i = 1}^{n + \alpha_0 + \beta_0} w_i } 
\end{align*}
To characterize the distribution of $\tth$, let us define $P_{0}$ and $P_{1}$ as the sum of weights for the positive and negative examples respectively. Formally,
\begin{align*}
P_{0} & = \sum_{i = 1}^{n} \left[ w_i \cdot \cI \{ y_i = 0 \} \right] + \sum_{i = 1}^{\alpha_0} [w_i]  \\
P_{1} & = \sum_{i = 1}^{n} \left[ w_i \cdot \cI \{ y_i = 1 \} \right] + \sum_{i = 1}^{\beta_0}  [w_i] 
\end{align*}
The sample $\tth$ can then be rewritten as:
\begin{align*}
\tth = \frac{P_1}{P_0 + P_1}
\end{align*}
Observe that $P_0$ (and $P_1$) is the sum of $\alpha + \alpha_0$ (and $\beta + \beta_0$ respectively) exponentially distributed random variables. Hence, $P_{0} \sim Gamma(\alpha + \alpha_0, 1)$ and $P_{1} \sim Gamma(\beta + \beta_0, 1)$. This implies that $\tth \sim Beta(\alpha + \alpha_0, \beta + \beta_0)$. 

When using the $Beta(\alpha_0,\beta_0)$ prior for TS, the corresponding posterior distribution on observing $\alpha$ positive examples and $\beta$ negative examples is $Beta(\alpha + \alpha_0, \beta + \beta_0)$. Hence computing $\tth$ according to \GWB is the equivalent to sampling from the Beta posterior. Hence, \GWB with multiplicative exponential weights is mathematically equivalent to TS.  
\end{proof}
\subsubsection{Categorical reward distribution}
\label{app:categorical}
\begin{proposition}
Let the rewards $y_i \sim Multinomial(\theta^{*}_1, \theta^{*}_2, \ldots \theta^{*}_C)$ where $C$ is the number of categories and $\theta^{*}_{i}$ is the probability of an example belonging to category $i$. In this case, weighted bootstrapping with $w_i \sim Exp(1)$ and the transformation $y_i \rightarrow y_i \cdot w_i$ results in $\tth \sim Dirichlet( n_1 + \tilde{n}_1, n_2 + \tilde{n}_2, \ldots n_{c} + \tilde{n}_c)$ where $n_i$ is the number of observations and $\tilde{n}_i$ is the pseudo-count for category $i$. In this case, \GWB is equivalent to Thompson sampling under the $Dirichlet(\tilde{n}_1,\tilde{n}_2, \ldots \tilde{n}_{C})$ prior.
\begin{proof}
Like in the Bernoulli case, for all $c \in C$, define $P_{c}$ as follows:
\begin{align*}
P_{c} & = \sum_{i = 1}^{n_c} \left[ w_i \cdot \cI \{ y_i = c \} \right] + \sum_{i = 1}^{\tilde{n}_c} [w_i]  
\end{align*}
The bootstrap sample $\tth$ consists of $C$ dimensions i.e. $\tth = (\tth_{1}, \tth_{2} \ldots \tth_{C})$ such that:
\begin{align*}
\tth_{c} = \frac{P_{c}}{\sum_{i = 1}^{C} P_{c}}
\end{align*}
Note that $\sum_{c = 1}^{C} \tth_{c} = 1$. Observe that $P_c$ is the sum of $n_c + \tilde{n}_c$ exponentially distributed random variables. Hence, $P_{c} \sim Gamma(n_c + \tilde{n}_c, 1)$. This implies that $\tth \sim Dirichlet(n_1 + \tilde{n}_{1}, n_2  +\tilde{n}_{2} \ldots n_k + \tilde{n}_{k})$. 

When using the $Dirichlet(\tilde{n}_{1}, \tilde{n}_2, \ldots \tilde{n}_{C})$ prior for TS, the corresponding posterior distribution is $Dirichlet(n_1 + \tilde{n}_{1}, n_2  +\tilde{n}_{2} \ldots n_k + \tilde{n}_{k})$. Hence computing $\tth$ according to \GWB is the equivalent to sampling from the Dirichlet posterior. Hence, \GWB with multiplicative exponential weights is mathematically equivalent to TS.  
\end{proof}
\end{proposition}
\subsection{Using additive normal weights}
In this subsection, we consider additive normal weights, implying that $w_i \sim N(0,1)$ and $\cT(y_i, w_i) = y_i + w_i$. We show that in this setting \GWB is mathematically equivalent to TS for Gaussian rewards. 
\subsubsection{Normal}
\label{app:normal}
\begin{proposition}
Let the rewards $y_i \sim Normal(\langle \bx_i, \theta^{*} \rangle, 1)$ where $\bx_{i}$ is the corresponding feature vector for point $i$. If $X$ is the $n \times d$ matrix of feature vectors and $\by$ is the vector of labels for the $n$ observations, then weighted bootstrapping with $w_i \sim N(0,1)$ and using the transformation $y_i \rightarrow y_i + w_i$ results in $\tth \sim N(\hth, \Sigma)$ where $\Sigma^{-1} = X^{T} X$ and $\hth = \Sigma \left[ X^{T} \by \right]$. In this case, \GWB is equivalent to Thompson sampling under the uninformative prior $\theta \sim N(0,\infty)$.
\begin{proof}
The probability of observing point $i$ when the mean is $\theta$ and assuming unit variance, 
\begin{align*}
\cP(y_i | \bx_i, \theta ) & = N( \langle \bx_i, \theta \rangle ,1) \\
\intertext{The log-likelihood for observing the data is equal to:}
\cL(\theta) &= \frac{-1}{2} \sum_{i = 1}^{n} \left( y_i - \langle x_i, \theta \rangle \right)^{2} \\
\intertext{The MLE has the following closed form solution:}
\hth & = \left( X^{T} X \right)^{-1} X^{T} \by
\intertext{The bootstrapped log-likelihood is given as:}
\tilde{\cL}(\theta) &= \frac{-1}{2} \sum_{i = 1}^{n} \left( y_i + w_i - \langle x_i, \theta \rangle \right)^{2} \\
\intertext{If $\bw = [w_1, w_2 \ldots w_n]$ is the vector of weights, then the bootstrap sample can be computed as:}
\tth & = \left( X^{T} X \right)^{-1} X^{T} \left[ \by + \bw \right]
\end{align*}
The bootstrap estimator $\tth$ has a Gaussian distribution since it is a linear combination of Gaussian random variables ($\by$ and $\bw$). We now calculate the first and second moments for $\tth$ wrt to the random variables $\bw$. 
\begin{align*}
\E[\tth] & = \E_{\bw} \left[ \left( X^{T} X \right)^{-1} X^{T} \left[ \by + \bw \right] \right] \\
& = \left( X^{T} X \right)^{-1} X^{T} \by + \E \left[ \left( X^{T} X \right)^{-1} X^{T} \bw \right] \\
& = \hth + \left( X^{T} X \right)^{-1} X^{T} \E[ \bw ] \\
\implies \E_{\bw}[\tth] & = \hth
\end{align*}
\begin{align*}
\E_{\bw} \left[ (\tth - \hth)(\tth - \hth)^{T} \right] & = \E_{\bw} \left[ \left[ (X^T X)^{-1} X^{T} \bw \right] \left[ (X^T X)^{-1} X^{T} \bw \right]^{T} \right]\\
& = \E \left[ \left[(X^T X)^{-1} X^{T} \bw \bw^{T} X (X^T X)^{-T} \right] \right] \\
& = \E \left[ (X^T X)^{-1} X^{T} \bw \bw^{T} X (X^T X)^{-T}  \right] \\
& = (X^T X)^{-1} X^{T}  E \left[ \bw \bw^{T} \right] X (X^T X)^{-T} \\
& = (X^T X)^{-1} X^{T} X (X^T X)^{-T} \tag{Since $\E[\bw \bw^{T}] = I_{d}$} \\
& = (X^T X)^{-1} (X^{T} X) (X^T X)^{-1} \\
\implies \E_{\bw} \left[ (\tth - \hth)(\tth - \hth)^{T} \right] & = (X^T X)^{-1} = \Sigma 
\end{align*}
Thus $\tth \sim N(\hth, \Sigma)$. When using the uninformative prior $N(0, \infty I_{d})$ prior for TS, the posterior distribution on observing $\cD$ is equal to $N(\hth, \Sigma)$. Hence computing $\tth$ according to \GWB is the equivalent to sampling from the the Gaussian posterior. Hence, \GWB with additive normal weights is mathematically equivalent to TS.  
\end{proof}
\end{proposition}
\section{Additional Experimental Results}
\label{app:additional-results}
\begin{figure*}[!ht]
\centering
        \subfigure[Bernoulli]
        {
			\includegraphics[width=0.23\textwidth]{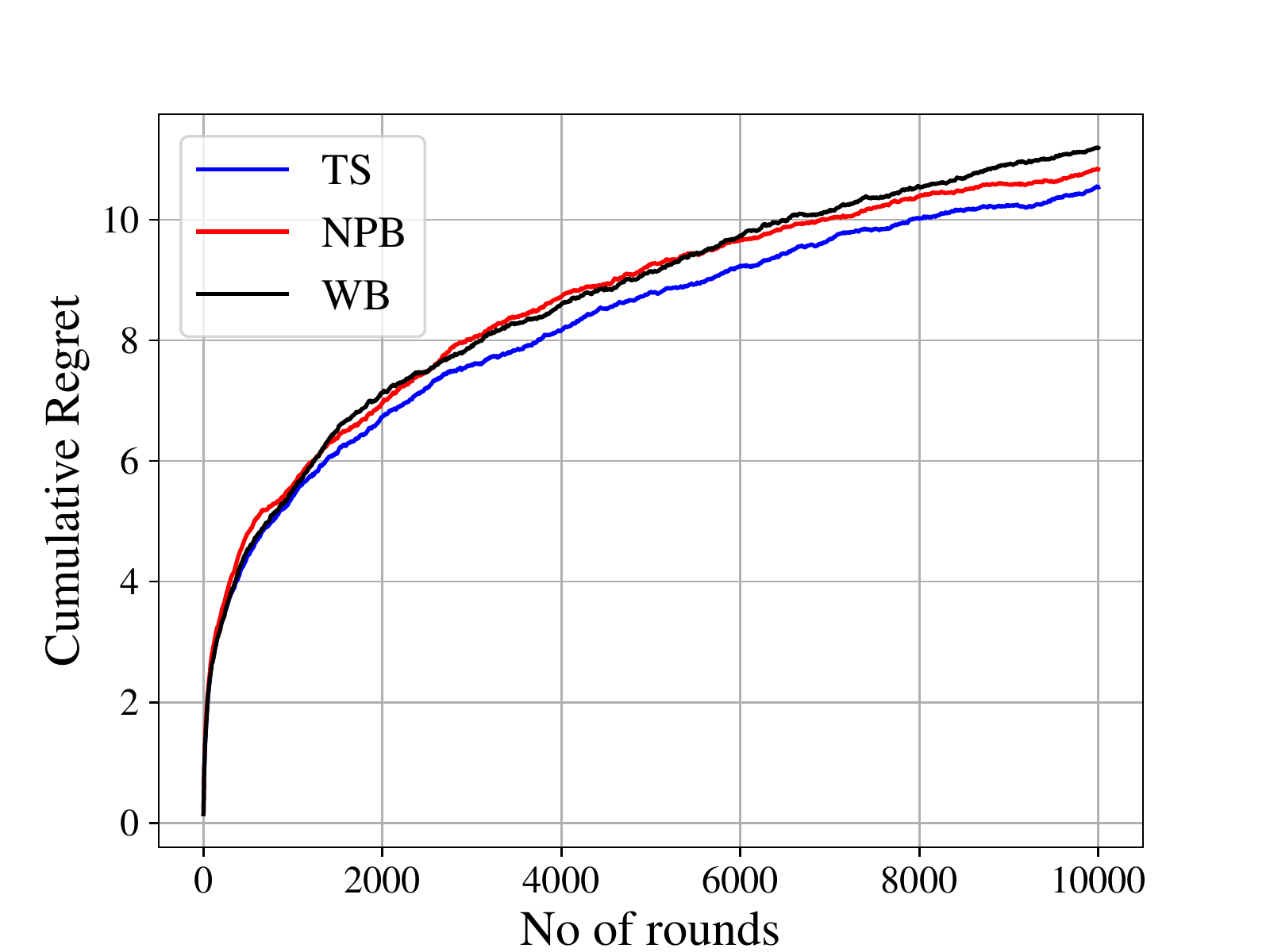}
			\label{fig:mab-bernoulli-2}
        }        
        \subfigure[Truncated-Normal]
        {
			\includegraphics[width=0.23\textwidth]{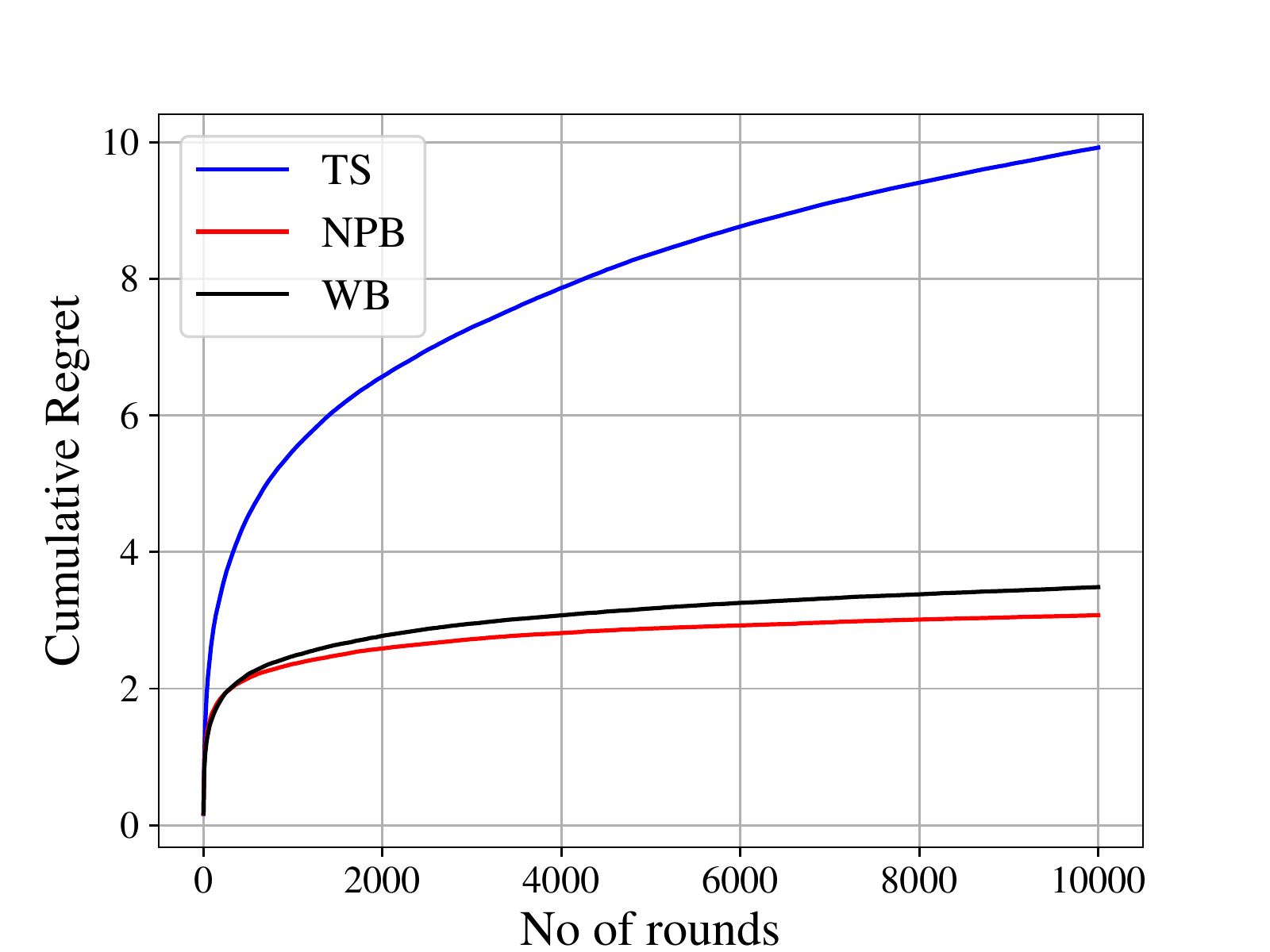}
			\label{fig:mab-truncated-normal-2}
        }              
       \subfigure[Beta]
        {
			\includegraphics[width=0.23\textwidth]{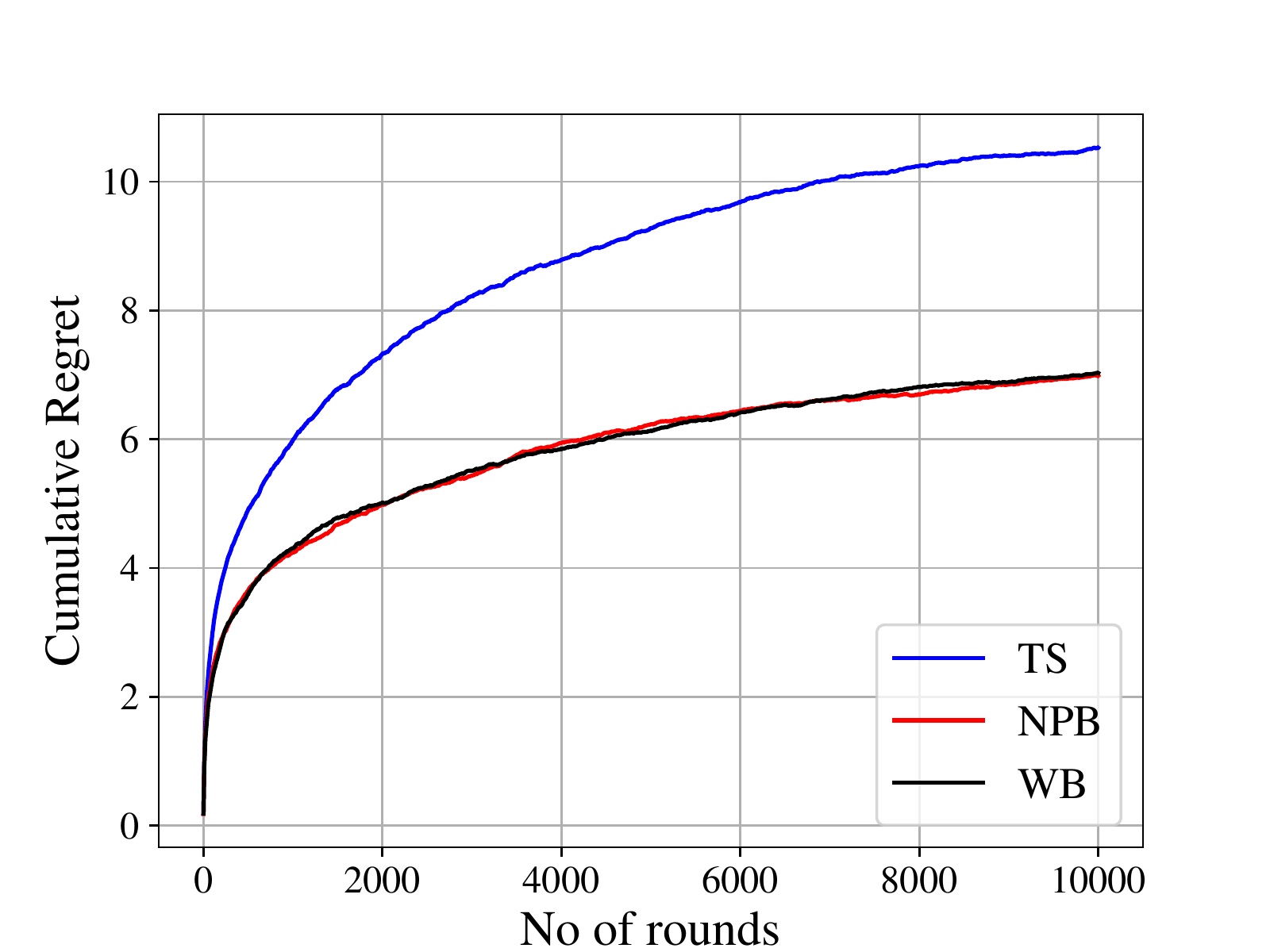}
			\label{fig:mab-beta-2}
        }              	
        \subfigure[Triangular]
        {
			\includegraphics[width=0.23\textwidth]{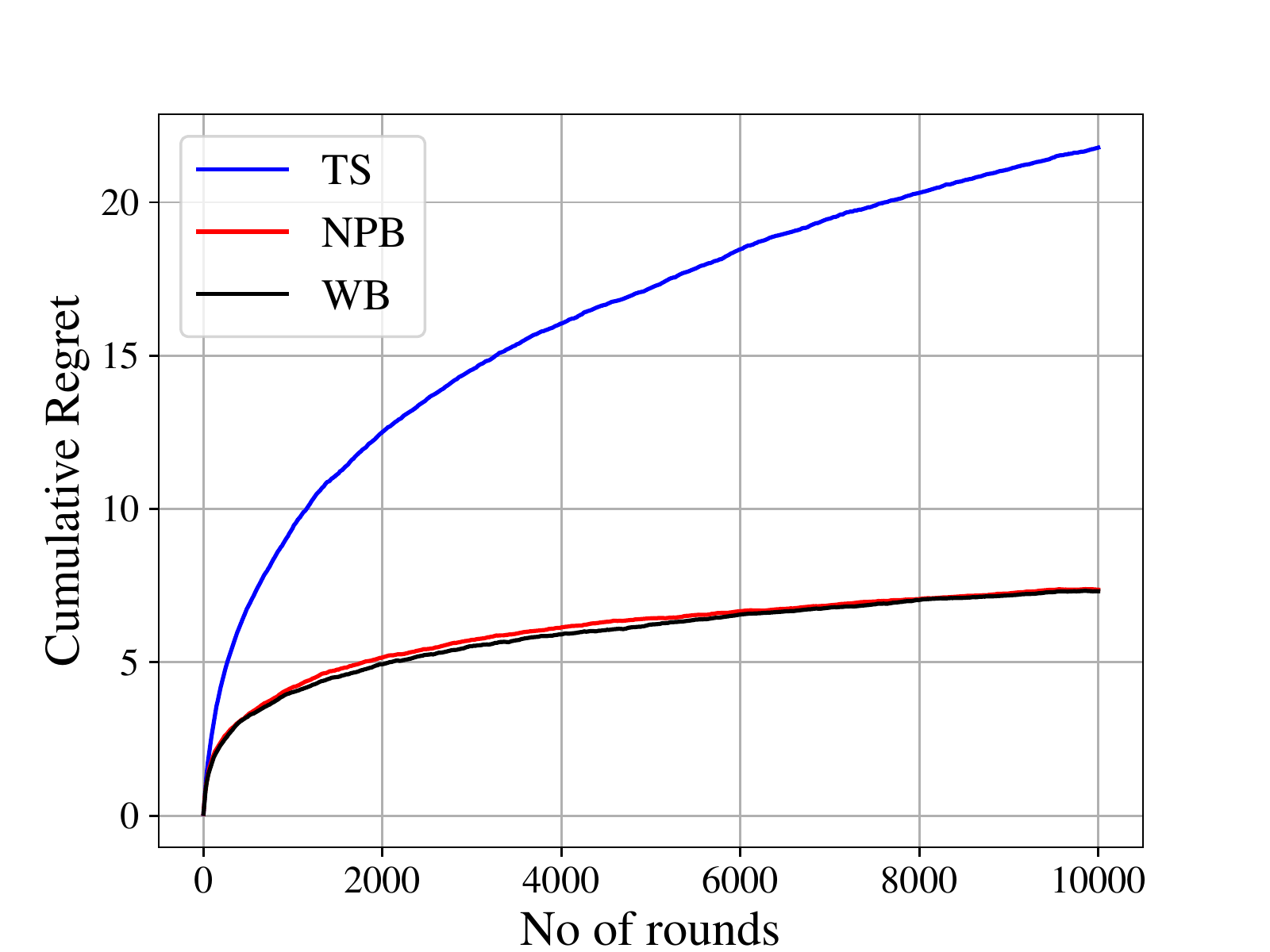}
			\label{fig:mab-triangular-2}
        }              	
      	\caption{Cumulative Regret for TS, \NPB and \GWB in a bandit setting $K = 2$ arms for (a) Bernoulli (b) Truncated-Normal in $[0,1]$ (c) Beta (d) Triangular in $[0,1]$ reward distributions}
\end{figure*}    
\begin{figure*}[!ht]
\centering
        \subfigure[Bernoulli]
        {
			\includegraphics[width=0.23\textwidth]{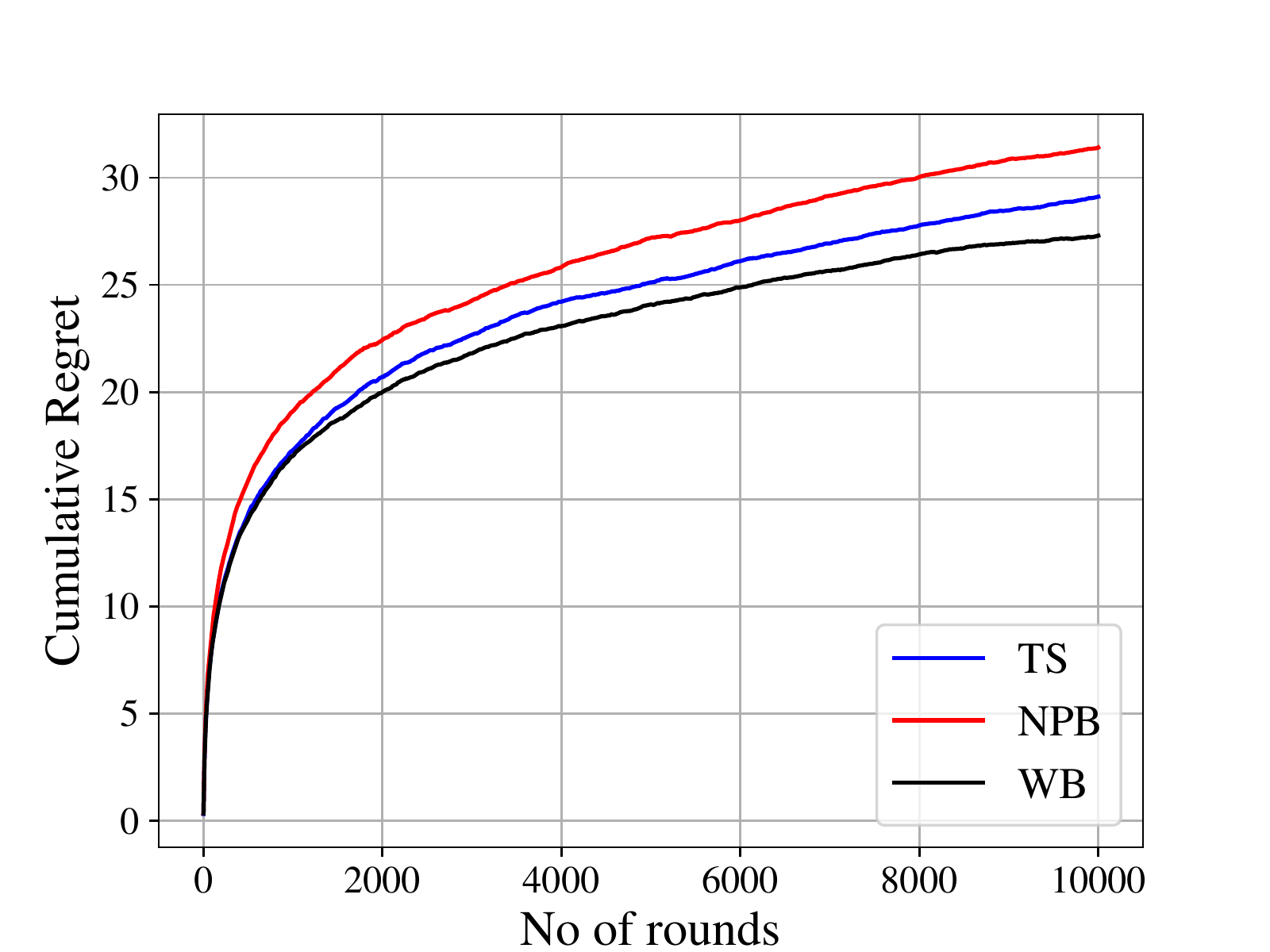}
			\label{fig:mab-bernoulli-5}
        }        
        \subfigure[Truncated-Normal]
        {
			\includegraphics[width=0.23\textwidth]{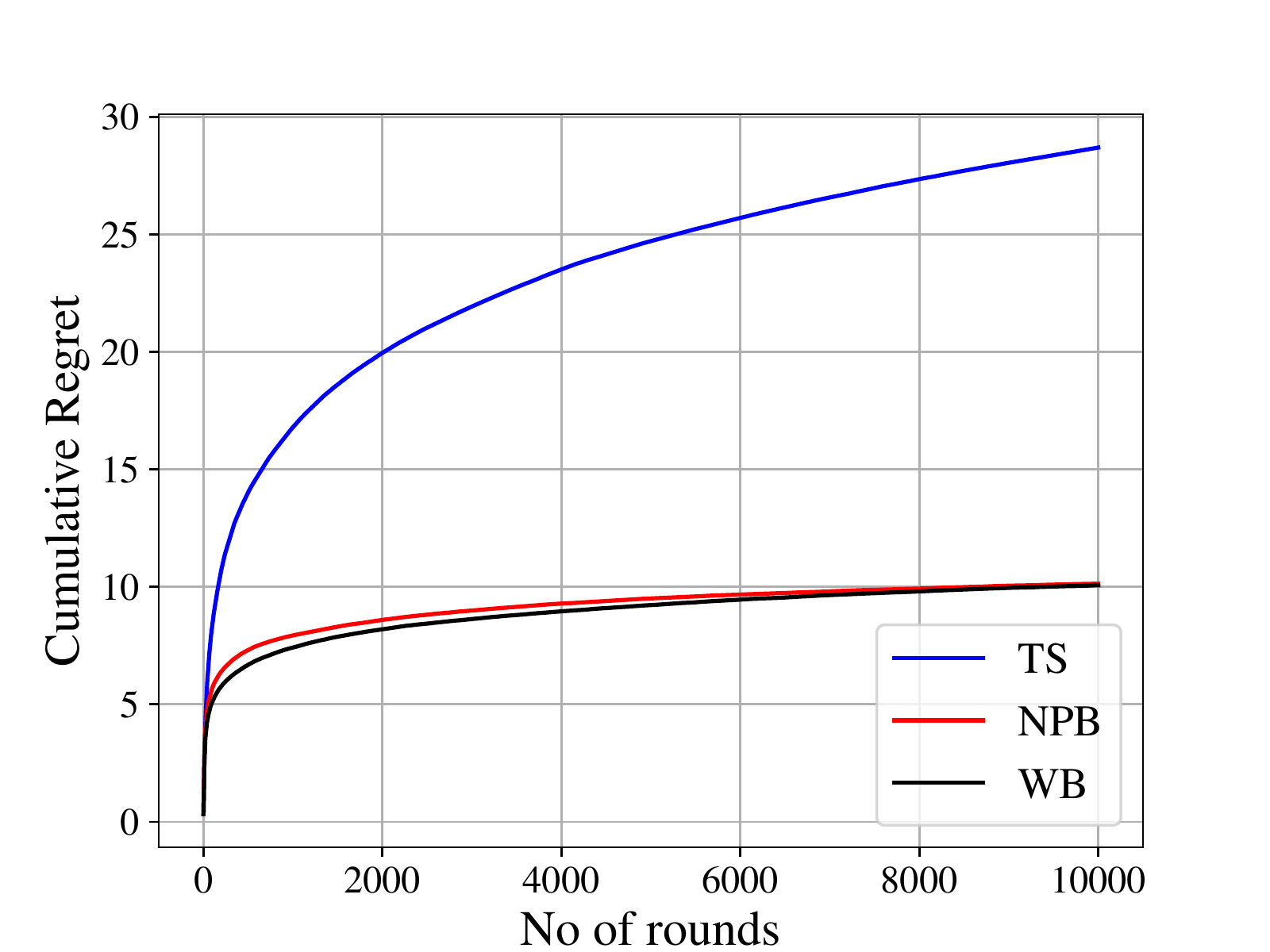}
			\label{fig:mab-truncated-normal-5}
        }              
       \subfigure[Beta]
        {
			\includegraphics[width=0.23\textwidth]{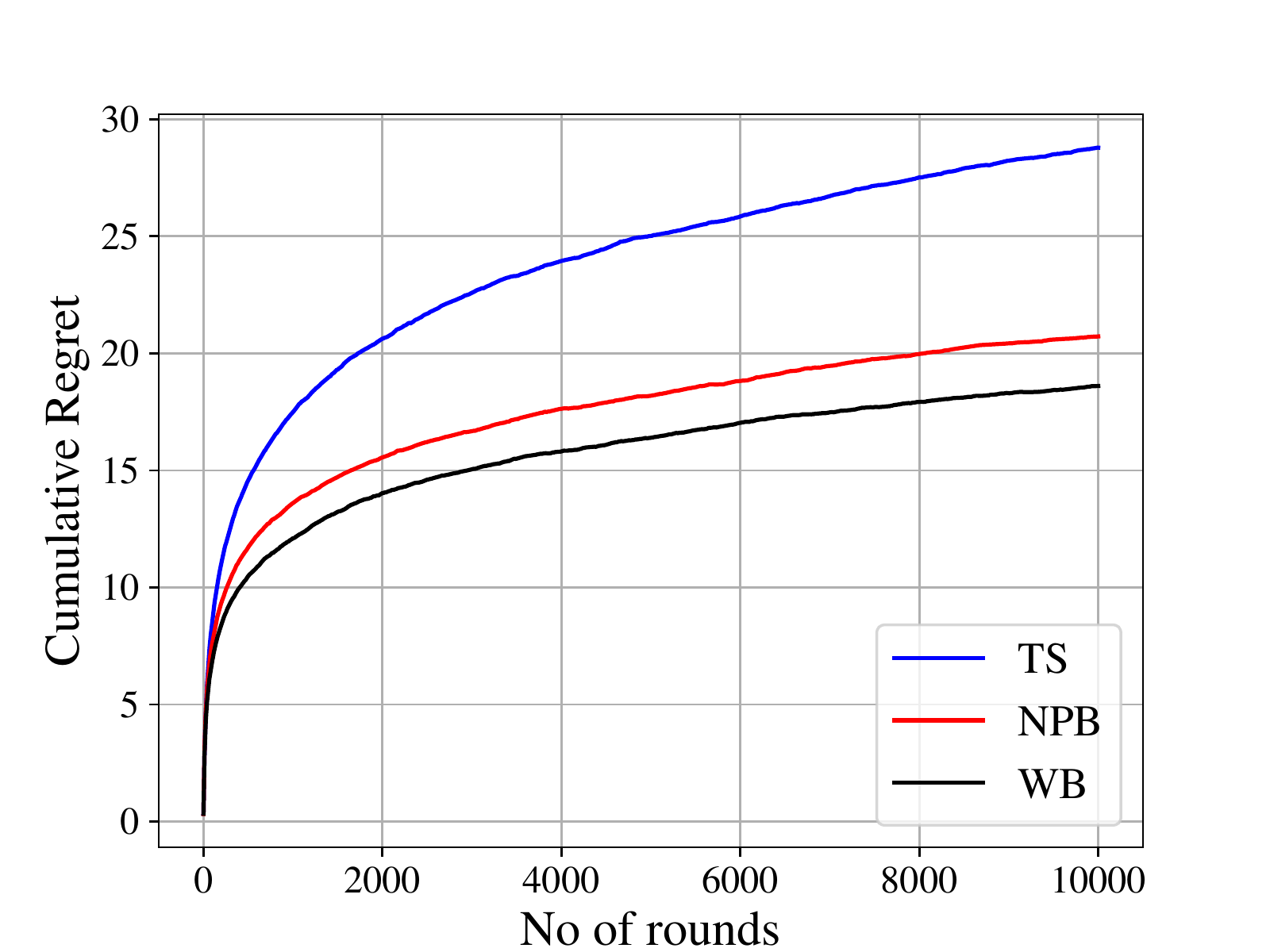}
			\label{fig:mab-beta-5}
        }              	
        \subfigure[Triangular]
        {
			\includegraphics[width=0.23\textwidth]{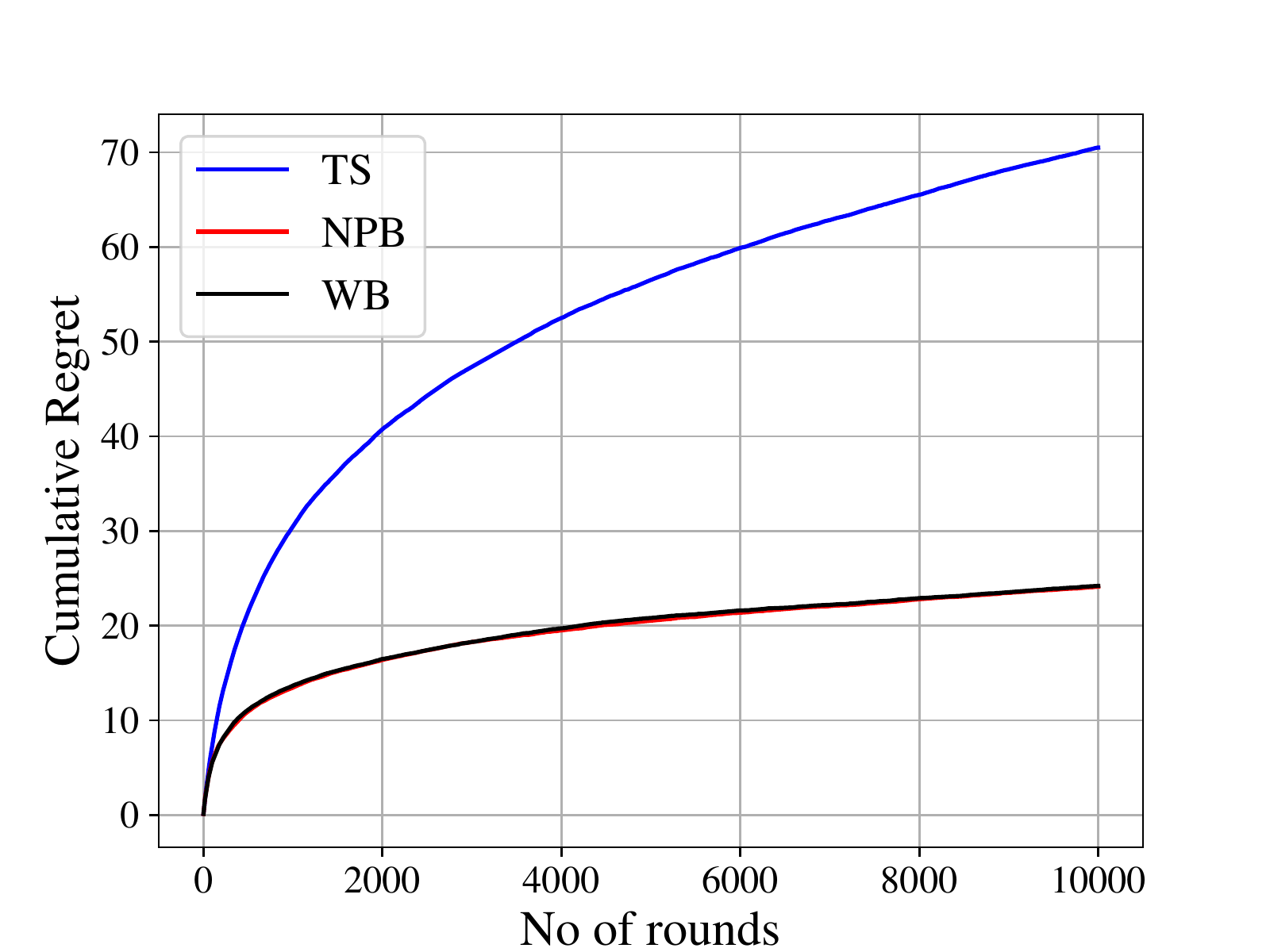}
			\label{fig:mab-triangular-5}
        }              	
 		\caption{Cumulative Regret for TS, \NPB and \GWB in a bandit setting $K = 5$ arms for (a) Bernoulli (b) Truncated-Normal in $[0,1]$ (c) Beta (d) Triangular in $[0,1]$ reward distributions}
\end{figure*}    

\end{document}